\documentclass[10pt,onecolumn, journal,cspaper,compsoc]{IEEEtran}

\usepackage{graphicx} 
\usepackage{subfigure}
\usepackage[square,sort,comma,numbers]{natbib}
\usepackage{algorithm2e}
\usepackage[noend]{algorithmic}
\usepackage{hyperref}
\usepackage{amsmath}
\usepackage{amsthm}
\usepackage{amssymb}
\usepackage{graphicx}
\usepackage{multirow}
\usepackage{url}
\usepackage{stfloats}
\usepackage{autobreak}
\allowdisplaybreaks

\newcommand{\tabref}[1]{Table~\ref{#1}}

\newtheorem{assumption}{Assumption}

\newtheorem{lemma}{Lemma}
\newtheorem{theorem}{Theorem}
\newtheorem{definition}{Definition}
\ifCLASSOPTIONcompsoc
\else
\fi
\ifCLASSINFOpdf
\else
\fi
\begin{document}
%
\title{SADA: A General Framework to Support Robust Causation Discovery with Theoretical Guarantee}
%
%
%
%

\author{Ruichu~Cai,~\IEEEmembership{Member,~IEEE,}
        Zhenjie~Zhang,~\IEEEmembership{Member,~IEEE,}
        and~Zhifeng~Hao,~\IEEEmembership{Member,~IEEE}
\IEEEcompsocitemizethanks{\IEEEcompsocthanksitem Ruichu Cai and Zhifeng Hao are with the Faculty of Computer Science, Guangdong University of Technology, Guangzhou, P.R. China, 510006. \protect \\ E-mail: \{cairuichu, zfhao\}@gdut.edu.cn \protect\\
\IEEEcompsocthanksitem Zhenjie Zhang is with Advanced Digital Sciences Center, Illinois at Singapore Pte. Ltd. Singapore, 138632. \protect \\ E-mail:zhenjie@adsc.com.sg}
\thanks{}}

%
%

\markboth{}%
{Cai \MakeLowercase{\textit{et al.}}: SADA: A General Framework to Support Robust Causation Discovery with Theoretical Guarantee}
%


\IEEEcompsoctitleabstractindextext{%

\begin{abstract}
Causation discovery without manipulation is considered a crucial problem to a variety of applications. The state-of-the-art solutions are applicable only when large numbers of samples are available or the problem domain is sufficiently small. Motivated by the observations of the local sparsity properties on causal structures, we propose a general \emph{Split-and-Merge} framework, named SADA, to enhance the scalability of a wide class of causation discovery algorithms. In SADA, the variables are partitioned into subsets, by finding \emph{causal cut} on the sparse causal structure over the variables. By running mainstream causation discovery algorithms as basic causal solvers on the subproblems, complete causal structure can be reconstructed by combining the partial results. SADA benefits from the recursive division technique, since each small subproblem generates more accurate result under the same number of samples. We theoretically prove that SADA always reduces the scales of problems without sacrifice on accuracy, under the condition of local causal sparsity and reliable conditional independence tests. We also present sufficient condition to accuracy enhancement by SADA, even when the conditional independence tests are vulnerable. Extensive experiments on both simulated and real-world datasets verify the improvements on scalability and accuracy by applying SADA together with existing causation discovery algorithms.
\end{abstract}


}

\maketitle

\IEEEdisplaynotcompsoctitleabstractindextext

%
\IEEEpeerreviewmaketitle

\section{Introduction}

Causation discovery plays an important role on a variety of scientific domains. Different from the mainstream statistical learning approaches, causation discovery tries to understand the data generation procedure, rather than characterizing the joint distribution of the observed variables only. It turns out that understanding causality in such procedures is essential to predict the consequences of interventions, which is the key to a large number of applications, such as genetic therapy, advertising campaign design, etc.


From computational perspective, causation discovery is usually formulated over a probabilistic graphical model on the variables under the assumption of faithfulness \cite{pearl2009causality}, in which the directed edges indicate causal relations. When it is unlikely to manipulate the samples in experiments, conditional independence tests are commonly employed to detect local causal structures among the variables \cite{pearl2009causality, spirtes2011}. Despite of the successes of these approaches on small problem domains and large sample bases, they usually fail to find true causalities, when huge number of equivalent structures over the graphical probabilistic models render exactly the same conditional independence.


To tackle the difficulties of causation discovery under non-experimental setting, researchers are recently resorting to asymmetrical relations between the cause-effect pairs under assumptions on the data generation process. The discovery ability is dramatically improved, by exploiting linear non-Gaussian assumption \cite{ShimizuJMLR06, ShimizuJMLR2011Direct}, nonlinear assumption \cite{Janzing2008nipsnonlinear}, discrete property \cite{Janzing2010jmlrdiscrete}, deterministic mechanism \cite{Daniusis2010Deterministic} and so on. When the variables are correlated under linear relations and the noises follow non-Gaussian distributions, for example, LiNGAM \cite{ShimizuJMLR06} and its variants \cite{ShimizuJMLR2011Direct} are known as the best causation discovery algorithms. However, the scalability of LiNGAM and its variants is still questionable, since they heavily depend on the independent component analysis (ICA) during the computation. To return robust results from ICA, it is necessary to feed a large bulk of samples, which are expected to be no smaller than the number of variables. Similar problems arise to other well known methods, e.g., \cite{Janzing2008nipsnonlinear}, \cite{Janzing2010jmlrdiscrete}, which are usually used to infer the causal directions over individual variable pairs.

Motivated by the common observations on the sparsity of causal structures, i.e., each variable usually only depends on a small number of parent variables, we derive a Scalable cAusation Discovery Algorithm (named SADA) in this article. SADA helps the existing causation algorithms to get rid of difficulties on small sample size in practice. Well designed conditional independence tests are conducted to partition the problem domain into small subproblems. With the same number of samples, existing causation discovery algorithm (taken as basic causal solvers) could generate more robust and accurate results on these small subproblems. Partial results from all subproblems are finally merged together, to return a complete picture of causalities among all the variables.

This framework is generic, as only faithfulness condition and causal sufficiency assumption are employed, so that it works well with different basic causal solvers by exploiting additional compatible data generation assumptions, such as linear non-Gaussian data and additive noise data. This framework is also theoretically solid, as it always returns correct and complete result under the optimal setting on conditional independence tests and basic causal solvers. Even when the conditional independence tests are vulnerable, the framework is capable of improving the recall and precision of the overall results, when the basic causal solver on the subproblems achieves sufficient enhancement on accuracy. Our experiments on synthetic and real datasets verify the genericity, superior scalability and effectiveness of our proposal, when applied together with two mainstream causation discovery algorithms.



The outline of the paper is listed as follows. Section \ref{sec:related} reviews existing studies on causation discovery problem. Section \ref{sec:framework} introduces the framework and algorithms to tackle the problem of small sample size. Section \ref{sec:theory1} provides theoretical analysis on the proposed framework with the assumption of ideal conditional independence tests. Section \ref{sec:theory2} extends the analysis to more general assumptions on vulnerable conditional independence tests. Section \ref{sec:exp} reports experimental results on both synthetic and real datasets, and Section \ref{sec:concl} finally concludes the paper.

\section{Related Work}\label{sec:related}

Causal Bayesian network (CBN) is part of the theoretical background of this work. Different from the traditional Bayesian network, each edge in a CBN is interpreted as a direct causal influence from the parent node to the child node \cite{pearl2009causality}. CBN has been used to model the causal structure in many real-world applications, e.g., the gene regulatory network \cite{yoo2002discovery,ellis2008learning} and causal feature selection \cite{aliferis_local_2010}.

A large number of works try to explore the conditional independence tests to learning the local structures of CBN, e.g. the well known PC algorithm \cite{kalisch2007pc,cai2011bassum}, Markov Blanket discovery methods \cite{zhu2007markov,spirtes2011}. These methods provide the elements of causal structures, and are usually considered as start points of the causation discovery methods \cite{pearl1991ic,aliferis_local_2010,cai2013causal}.

Pearl is one of the pioneers of the causal theory \cite{pearl2009causality}. Since Pearl's Inductive Causality \cite{pearl1991ic}, a large number of extensions are proposed by exploring the $V$-structure to determine the causal directions. Most of the extensions assume the acquisition of a sufficiently large sample set \cite{aliferis_local_2010}. Though there are studies aiming at the causation discovery under small sample size \cite{bromberg_improving_2009}, the actual number of the samples used in their empirical evaluations remains significantly larger than the number of variables. Cai's study \cite{cai2013causal} is another attempt under this category to extend the method to the high dimensional gene expression domain by exploiting the conflict relations among the local sub-structures. Recently, some partition based approaches are also proposed to improve the scalability of the structure learning based methods, such as Geng's recursive decomposition strategy \cite{XieJMLR08} and Yehezkel's autonomy identification based partition \cite{YehezkelJMLR09}. However, all these approaches, based on conditional independence tests, cannot distinguish two causality structures if they come from a so-called \emph{Markov equivalence class} \cite{pearl2009causality}, in which expensive intervention experiments are previously considered essential \cite{he_active_2008}.

Recently, a lot of methods are proposed to break the limitations of the methods purely under conditional independence tests, by exploiting the asymmetric properties in the generative progress, which brings a gleam of dawn to resolve the causal equivalence problem. Existing studies on this line can be categorized based on the adopted assumption on the noise type or data generation mechanism. \emph{Additive Noise Model} \cite{Janzing2008nipsnonlinear} and its variants highly depend on the independence relation between the causal variable and the noises, including, its generalization to post-linear \cite{zhang2009postnonlinear}, its variants on the discrete data \cite{Janzing2010jmlrdiscrete}. Information Geometry based method is developed for the deterministic causal relations \cite{Janzing2012AIInfoGeo} by exploring the asymmetric relation between the data distribution and the generation mechanism. Its extension exploits the Kernel Hilbert space embedding based measure to infer the asymmetric properties \cite{ChenZCS2014}. LiNGAM and its variants \cite{ShimizuJMLR06,ShimizuJMLR2011Direct}, assume that the data generating process is linear and the noise distributions are non-Gaussian. There are other studies relat to this topic, such as explaining the underlying theoretical foundation behind asymmetric property based methods \cite{Bastian2010COLT,Janzing2010Causal,Janzing2012AIInfoGeo}, addressing the latent variable problem\cite{tashiro2014parcelingam}, regression-based inference method \cite{Mooij2009ICMLRegression} and kernel independence test based causation discovery methods \cite{Zhang2012CoRRKCIT}. Inference the direction between a causal-effect pair is focus of these methods. Though, there are some works try to generalize the model to the case with more than two variables, for example \cite{ShimizuJMLR06,peters2012identifiability}, there is no existing work to address the sample size problem to the best of our knowledge.

Granger's causality \cite{Granger1969Econometrica} is another important subfield of causality, which uses Granger's causality test \cite{hacker2006tests} to determine whether one time series is useful in forecasting another. Recently, Granger' causality is extended to infer the gene regulatory networks from the time series gene expression data \cite{mukhopadhyay2007causality,lozano2009grouped}. Granger's work differs from traditional causation discovery techniques on two aspects. Firstly, compared with the conventional definition of causality, Granger's causality is more likely a regression method and does not reflect the true causality mechanism. Secondly, the temporal information is essential for Granger's causation discovery algorithms, which is expensive and some times impossible to collect.

\section{SADA Framework}\label{sec:framework}

\subsection{Preliminaries}

Assume that all samples from the problem domain contain information on $n$ different variables, i.e., $V=\{v_1,v_2\dots,v_n\}$. Let $D=\{x_1,x_2,\cdots, x_m\}$ denote an observation sample set. Each sample $x_i$ is a vector $x_i=\left(x_{i1},x_{i2},\ldots,x_{in},y_i\right)$, where $x_{ij}$ indicates the value of the sample $x_i$ on variable $v_j$ and $y_i$ is the target variable under investigation.

If $\mathcal{P}$ is a distribution over the domain of variables in $V$, we assume that there exists a causal Bayesian network $N$ faithful to the distribution $\mathcal{P}$. The network $N$ includes a directed acyclic graph $G$, each edge in which indicates a dependent relation between two variable nodes. Each edge is also associated with a conditional probability function which presents conditional probability distribution of the variables given the values of their parent variables. Following the common assumption of existing studies, we only consider problem domain meeting \emph{Faithfulness Condition} \cite{koller2009pgm}. Specifically, $\mathcal{P}$ and $N$ are faithful to each other, \emph{iff} every conditional independence entailed by $N$ corresponds to some Markov condition present in $\mathcal{P}$. Beside the faithfulness condition, \emph{Causal Sufficiency} \cite{koller2009pgm} is another assumption taken in this work, which assumes that there are no latent confounders of any two observed variables.

Due to the probabilistic nature, it is likely to find a huge number of equivalent Bayesian networks. Two different Bayesian Networks, $N_1$ and $N_2$, are Markov equivalent, if $N_1$ and $N_2$ entail exactly the same conditional independence relations among the variables. In all these Bayesian networks, Causal Bayesian network (CBN) is a special one in which each edge is interpreted as a direct causal relation between a parent node and a child node.

Generally speaking, it is difficult to distinguish CBN from independence equivalent Bayesian networks, unless additional assumptions are made. When the variables are correlated in linear relations and the noises follow non-Gaussian distributions independently, LiNGAM and its variants \cite{ShimizuJMLR06,ShimizuJMLR2011Direct} are known to return more accurate causations from uncontrollable samples. In particular, such assumption can be formulated by an equation, such that every variable $v_i = \sum_{v_j \in P\left(v_i\right)}{ A_{ij}\cdot v_j} + e_i$, where $P\left(v_i\right)$ contains all the parent variables of $v_i$, $A_{ij}$ is the linear dependence weight w.r.t. $v_i$ and its parent $v_j$, and $e_i$ is an non-Gaussian noise over $v_i$. Assume that the variables in $V$ are organized based on a topological order in the causal structure. The generation procedure of a sample could be written as $V = A\cdot V + E$. LiNGAM aims to find such a topological order and reconstructs the matrix $A$ by exploiting \emph{independence component analysis} (ICA) over the sample.

When assuming non-linear generation procedure \cite{Janzing2008nipsnonlinear} and discrete data domain \cite{Janzing2010jmlrdiscrete}, additive noise model provides another approach to utilize the asymmetric relations between causal variables and consequence variables. A regression model $v_i = f\left(v_j\right) + e_i$ is trained for each pair of variables $v_i$ and $v_j$. If the noise variable $e_i$ is independent of $v_j$, variable $v_j$ is returned as the cause of variable $v_i$. Note that algorithms under additive noise model are usually run over pairs of variables independently.

%

A common observation on the CBNs in real-world domains is the sparsity on the causal relations. Specifically, a variable usually only has a small number of causal variables in the CBN, regardless of the underlying true generative procedure. This property, however, is not fully exploited by the existing causation algorithms.


\subsection{Framework}
In SADA, the variables are partitioned into subsets, by utilizing \emph{causal cuts} on the variables based on conditional independence relations over the domain with sparse causal structure. To begin with, we present the definitions of \emph{causal cut} and \emph{causal cut set}.

\begin{definition}\emph{Causal Cut.}
Let $G=\left(V,E\right)$ denote a causal structure on the variable set $V$. Three disjoint variable subsets $\left(C,  V_1, V_2\right)$ of $V$ forms a \emph{causal cut} over $G$, \emph{if} (1) $C \cup V_1 \cup V_2 =V$; (2) there is no edge between $V_1$ and $V_2$ in $E$.
\end{definition}

\begin{definition}\emph{Causal Cut Set.} In a causal cut $\left(C,  V_1, V_2\right)$, the variable set $C$ is a \emph{causal cut set} if it ensures there is no edge between $V_1$ and $V_2$.
\end{definition}

Based on the above definitions, given a causal cut $\left(C,  V_1, V_2\right)$ over the problem $G=\left(V,E\right)$, one of the two following cases must hold for each directed edge $u \rightarrow v$ in $E$: (1) intra-causality: $\{u,v\}\subset V_1$, $\{u, v\}\subset V_2$ or $\{u, v\}\subset C$; and (2) inter-causality: $u \in V_1\cup V_2$ and $v \in C$, or $u \in C$ and $v \in V_1\cup V_2$. This intuition guarantees the independence between two subproblems on the variable sets $V_1 \cup C$ and $V_2 \cup C$, which paves the foundation for the 'Split-and-Merge' framework.

\begin{figure}[t]
\centering
\includegraphics[width=.6\columnwidth]{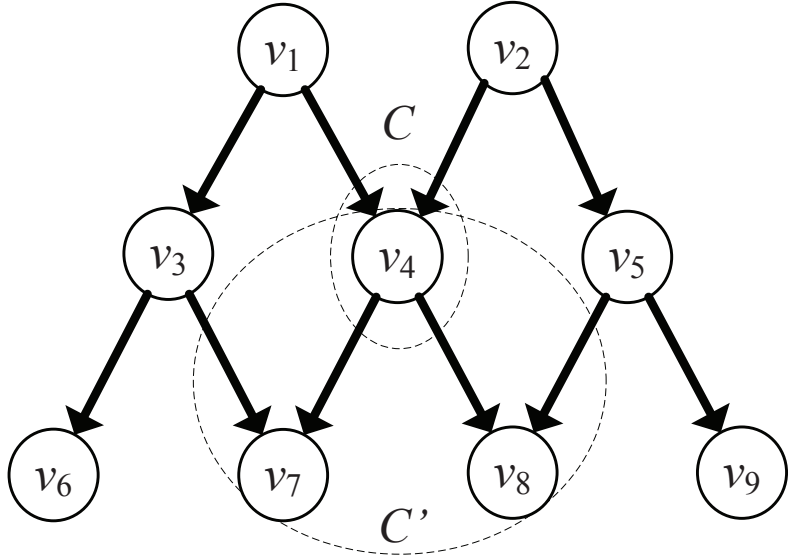}
\caption{An example probabilistic graphical model over 9 variables and two causal cuts with causal cut sets $C$ and $C'$.} \label{fig:example}
\end{figure}

In Figure \ref{fig:example}, for example, $C=\{v_4\}$ is a causal cut set, which separates the variables into the causal cut $\left(C=\{v_4\}, V_1=\{v_1, v_3, v_6,v_7\}, V_2=\{ v_2, v_5, v_8, v_9\}\right) $. Given a directed graph $G$, there could be different valid causal cut sets satisfying the above conditions. In the example graph, $C'=\{v_4, v_7, v_8\}$ is another causal cut set with $V_1=\{v_1, v_3, v_6\}$ and $V_2=\{ v_2, v_5, v_9\}$.

Please note that causal cut set is closely related to the concept of $d$-separation but it may not lead to $d$-separation. For example, in the causal cut $\left(C=\{v_4\}, V_1=\{v_1, v_3, v_6,v_7\}, V_2=\{ v_2, v_5, v_8, v_9\}\right)$ , the variable $v_7$ is not independent of $v_2$ given any subset of the causal cut set $C=\{v_4\}$; while in the  causal cut $\left( C'=\{v_4, v_7, v_8\}, V_1=\{v_1, v_3, v_6\}, V_2=\{ v_2, v_5, v_9\}\right)$, the causal cut set $C'$ lead to the $d$-separation between $V_1$ and $V_2$. The connection between causal cut and $d$-separation will be formalized in the next section.

Given a causal cut $\left(C,V_1,V_2\right)$ on variable set $V$, we are able to transfer the causation discovery problem on $V$ into two smaller ones over the variable sets $V_1\cup C$ and $V_2\cup C$ respectively. This partitioning operation could be recursively called, until the number of variables involved in the subproblem is below a specified threshold $\theta$. The complete pseudocodes are available in Algorithm \ref{algo:framework}. The inputs of SADA include the sample set $D$, the variables $V$, a threshold $\theta$ and an underlying causation discovery algorithm $A$. Here, $\theta$ is used to terminate the recursive partitioning when the variable set is sufficiently small, and $A$ is an arbitrary basic causal solver invoked to find the actual causal structure on the subset of variables. $A$ is usually taken as a basic causal solver in this work.

In the rest of the section, we will discuss how to effectively and efficiently find causal cut on a variable set $V$ by exploiting the corresponding observational samples. We will also present the details of the merging operator, which tackles the problem of inconsistency and redundancy on the partial results from the subproblems. Note that only the faithfulness condition and causal sufficiency assumption are employed in this proposed framework (except the basic causal solver), thus SADA can work with different basic causal solvers by exploiting additional data generation assumptions, such as linear non-Gaussian data, additive noise data and so on.

\begin{algorithm}[t]
\caption{\textbf{SADA}}\label{algo:framework}
\begin{algorithmic}
   \STATE Input: sample set $D$, variable set $V$, variable threshold $\theta$ and a basic causal solver $A$\\
   \STATE Output: $G$: causal structure \\
   \IF{$|V|\leq \theta$}
     \STATE Return the result $G$ by running algorithm $A$ on $D$ and $V$.
   \ENDIF
   \STATE Find a causal cut $\left(C,V_1,V_2\right)$ on $D$ and $V$.
   \STATE $G_1=$\textbf{SADA}$\left(D,V_1\cup C,\theta,A\right)$.
   \STATE $G_2=$\textbf{SADA}$\left(D,V_2\cup C,\theta,A\right)$.
   \STATE Return $G$ by merging $G_1$ and $G_2$.
\end{algorithmic}
\normalsize
\end{algorithm}

\subsection{Finding Causal Cut}

The searching of the causal cut is crucial to the partitioning operation in SADA. To identify potential causal cut, our algorithm resorts to conditional independence testing over variables in the Bayesian network. The following lemma formalizes the connection.

\begin{lemma} \label{lemma:properyofcvs}
$\left(C,V_1,V_2\right)$ is a causal cut over causal structure $G$, \emph{if} (1) $C \cup V_1\cup V_2 =V$; (2) $C$, $V_1$ and $V_2$ are disjoint with each other; and (3) $\forall u\in V_1$ and $\forall v \in V_2$, there exists a variable set $C_{uv} \subset C$ such that $u \bot v | C_{uv}$.
\end{lemma}

\begin{proof} For all pairs of variables $\left(u,v\right)$ that $u\in V_1$ and $v \in V_2$ are \emph{d}-separated by $C$, there is no directed edge between $V_1$ and $V_2$. Combining condition (1) and (2), we have that $\left(C,V_1,V_2\right)$ is a causal cut over causal structure $G$. \end{proof}

Note that Lemma \ref{lemma:properyofcvs} is a sufficient condition of causal cut, but not a necessary condition. That is, a causal cut $\left(C,V_1,V_2\right)$ may not satisfy the above three conditions. For example, the triple  $\left(C=\{v_4\}, V_1=\{v_1, v_3, v_6,v_7\}, V_2=\{ v_2, v_5, v_8, v_9\}\right) $ is a causal cut of the example given in Figure \ref{fig:example}, but it dose not satisfy the condition (3) of Lemma \ref{lemma:properyofcvs}, because the variable $v_7$ is not independent of $v_2$ given any subset of the causal cut set $C=\{v_4\}$. In detail, when $C_{uv}=\emptyset$, $v_7$ is dependent on $v_2$ because of the directed path $v_2 \rightarrow v_4\rightarrow v_7$; when $C_{uv}=\{v_4\}$, $v_7$ is dependent on $v_2$ because of the directed path $v_2 \rightarrow v_4\leftarrow v_1 \rightarrow v_3 \rightarrow v_7$ ( the path $v_2 \rightarrow v_4\leftarrow v_1$ is connected given the variable $v_4$).

By exploiting the sufficient condition given in Lemma \ref{lemma:properyofcvs}, we derive a new algorithm to find a causal cut set $C$, and the corresponding causal cut $\left( C, V_1, V_2 \right)$. In the search algorithm, each variable is heuristically assigned to one of the set $V_1$, $V_2$ and $C$. Besides of casual cut property, we also want to optimize the following two objectives during the assignment procedure: (1)minimizing the size of $C$. Because $C$ appears in both subproblems $V_1 \cup C$ and $V_1 \cup C$, i.e., smaller $C$ is preferred; (2) minimizing the size difference of $V_1$ and $V_2$. According to the principle of divide-and-conquer, the causal cut with similar sizes of $V_1$ and $V_2$ is preferred.
%
%

The details of the algorithms are listed in Algorithm \ref{algo:split}. The algorithm runs with $k$ different initial variable pairs. The algorithm greedily adds the variable $w$ into $V_1$ (or $V_2$), if $w$ is independent of all the variables of $V_2$ (or $V_1$) given some subset of $C$. Only the $w$ can not added to neither of $V_1$ and $V_2$, $w$ is added to $C$. After completing all assignments, the algorithm also tries to move the variables from $C$ to $V_1$ or $V_2$ to maximize the partitioning effect. Finally, the causal cut with largest $\min\{|V_1|,|V_2|\}$ are returned as final result. We leave the discussion on the parameters $k$ and $\theta$ to next section.
Please note that the sample size needed in the cut algorithm highly depends on the local connectivity of the causal structure but not on the number of variables. This is an important advantage of the algorithm to applications in large scale sparse causation discovery problems.

\begin{algorithm}[t]
\caption{\textbf{Finding Causal Cut}}\label{algo:split}
\begin{algorithmic}
   \STATE Input: sample set $D$, variable set $V$, number of initial variable pairs $k$
   \STATE Output: a causal cut $\left(C,V_1,V_2\right)$
   \FOR{$j=1$ to $k$}
   \STATE Randomly pick up two variables $u$ and $v$ such that $\exists V'\subset V-\{u,v\}$ satisfies $u\bot v| V'$.
   \STATE Find the smallest $\hat{V}\subseteq V-\{u,v\}$ to make $u\bot v | \hat{V}$.
   \STATE Initialize $V_1=\{u\}$, $V_2=\{v\}$ and $C=\hat{V}$.
   \STATE Remove variables in $V_1$, $V_2$ and $C$ from $V$.
   \FOR{each variable $w\in V$}
       \IF{ $\forall u \in V_1$, $\exists C'\subseteq C$ that $w \bot u |C'$}
         \STATE  Add $w$ into $V_2$.
       \ELSIF{$\forall v\in V_2$, $\exists C'\subseteq C$ that $w \bot v |C'$ }
         \STATE Add $w$ into $V_1$.
       \ELSE
         \STATE Add $w$ into $C$.
       \ENDIF
   \ENDFOR
   \FOR{each variable $s\in C$}
       \IF{$\forall u\in V_1$, $\exists C'\subseteq C-\{s\}$ that $s \bot u |C'$ }
          \STATE Move $s$ from $C$ to $V_2$.
       \ELSIF{$\forall v\in V_2$, $\exists C'\subseteq C-\{s\}$ that $s \bot v |C'$ }
          \STATE Move $s$ from $C$ to $V_1$.
       \ENDIF
   \ENDFOR
   \STATE Let $\Phi_j=\left(C,V_1,V_2\right)$
   \ENDFOR
   Return $\Phi_j$ with the largest $\min\{|V_1|,|V_2|\}$.
\end{algorithmic}
\normalsize
\end{algorithm}

Given the example structure of Figure \ref{fig:example}, the running step is given in Table \ref{tab:divide} under the assumption that $V_1$, $V_2$ and $C$ are initialized as $\{v_1\}$ , $\{v_2\}$ and $\emptyset$, respectively. Among the steps, $v_3$ and $v_6$ are marginally independent of any variable of the current $V_2$. In another word, $C'=\emptyset$ is used in the conditional independence test. The similar cases happen in the assignment of $v_5$ and $v_9$. In the checking of $v_7$, $v_7$ is dependent of the variables of $v_2$ given any sub set of the current causal cut set $C=\{v_4\}$ and added to the causal cut set $C$. $v_8$ is processed similarly.

\begin{table}[!h]
\center \caption{Running Example of Split}
\label{tab:divide}
\begin{tabular}{|c|c|c|c|c|}
\hline Step & $V$ & $V_1$ & C & $V_2$\\
\hline Initial & $v_3,v_4,v_5,v_6, v_7,v_8,v_9$ & $v_1$ & $\emptyset$ & $v_2$\\
\hline Check $v_3$ & $v_4,v_5,v_6, v_7,v_8,v_9$ & $v_1,v_3$ &  $\emptyset$ & $v_2$\\
\hline Check $v_4$ & $v_5,v_6, v_7,v_8,v_9$ & $v_1,v_3$ &  $v_4$ & $v_2$\\
\hline Check $v_5$ & $v_6, v_7,v_8,v_9$ & $v_1,v_3$ &  $v_4$ & $v_2,v_5$\\
\hline Check $v_6$ & $v_7,v_8,v_9$ & $v_1,v_3,v_6$ &  $v_4$ & $v_2,v_5$\\
\hline Check $v_7$ & $v_8,v_9$ & $v_1,v_3,v_6$ &  $v_4,v_7$ & $v_2,v_5$\\
\hline Check $v_8$ & $v_9$ & $v_1,v_3,v_6$ &  $v_4,v_7,v_8$ & $v_2,v_5$\\
\hline Check $v_9$ & $\emptyset$ & $v_1,v_3,v_6$ &  $v_4,v_7,v_8$ & $v_2,v_5,v_9$\\
\hline
\end{tabular}
\end{table}

\subsection{Merging Partial Results}\label{sec:merge}

As is shown in Algorithm \ref{algo:framework}, two partial results $G_1$ and $G_2$ are combined as a single casual graph as on variables in $V$. Since $G_1$ and $G_2$ are calculated independently but contain overlap over the causal cut set $C$. Thus conflict and redundancy need to be carefully handled in the merging operation. Recall the example given in Figure \ref{fig:example},  assume $C=\{v_4, v_7, v_8\}$, $V_1=\{v_1, v_3, v_4, v_6, v_7\}$ and $V_2=\{v_2, v_4, v_5, v_8, v_8, v_9\}$, the edges $v_7\rightarrow v_8$ and $v_8\rightarrow v_7$ may be appear in results from $V_1$ and $V_2$, generating conflicts. Similarly, the basic solver may return $v_4\rightarrow v_7$ and $v_7 \rightarrow v_8$ on $V_1$, and $v_4\rightarrow v_8$ on $V_2$. It is easy to see, the edge $v_7\rightarrow v_8$ is redundant. Generally speaking, such conflicts and redundancy depend on the assumption of the causal structures. Under the directed acyclic graph assumption, the examples shown in Fig 2 are all the patterns we can detect.

The general form of a conflict is a cycle of directed edges among a group of variables, as shown in Figure \ref{fig:merge_conflict}. Given two nodes $v_1$ and $v_2$, there are two paths co-existing, such as $v_1 \rightarrow  \cdots \rightarrow v_2$ and $v_1 \leftarrow v_2$. These two paths form a cycle and violate the acyclic constraints. To resolve such conflict, we simply remove the least reliable edge in the cycle, whenever a cycle is found. Here the reliability of the edge $v_1\rightarrow v_2$ is measured by the significance level, $sig\left(v_1\rightarrow v_2\right)$, which is returned by the basic causal solvers. For example, the $p$-value of the Wald test is used as the significance level for edges returned by LiNGAM \cite{ShimizuJMLR06}, and the $p$-value of the noise's independence of the causal variable is used as the significance level for edges returned by additive noise model \cite{Janzing2011tpamidiscrete}.

Figure \ref{fig:merge_redundancy} illuminates a potential redundancy case. Given two variables $v_1$ and $v_2$, if both $v_1 \rightarrow  \cdots \rightarrow v_2$ and $v_1 \rightarrow v_2$ are discovered, $v_1 \rightarrow v_2$ may be redundant. Because the dependency relation $v_1 \rightarrow v_2$ could be blocked by certain variables in the variable set $Path\left(v_1\rightarrow v_2\right)$. Here, $Path\left(v_1\rightarrow v_2\right)$ refers to the variable set involved in the directed path $v_1 \rightarrow  \cdots \rightarrow v_2$. Such redundancy raises when the following two conditions are satisfied: (1) the source and destination variables are both in the causal cut set, i.e., $v_1, v_2 \in C$, (2) there is another variable set $V_3\subset V_1$ (or $V_3\subset V_2$), such that $v_1 \rightarrow V_3 \rightarrow v_2$. If the above two conditions are met, one path $v_1 \rightarrow v_2$ will be returned from the subproblem over $V_1 \cup C$, while another path $v_1 \rightarrow v_2$ turns up from the other subproblem over $V_2 \cup C$. To tackle this problem, our merging algorithm runs the following conditional independence tests to verify if $\exists V' \subset Path\left(v_1\rightarrow v_2\right)$ that $v_1 \bot v_2 | V'$.

\begin{figure} [h]
  \centering
  \subfigure[Conflict]{
    \includegraphics[width=0.21\textwidth]{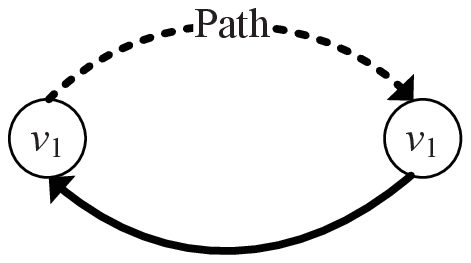}
    \label{fig:merge_conflict}
  }
  \subfigure[Potential Redundancy]{
    \includegraphics[width=0.21\textwidth]{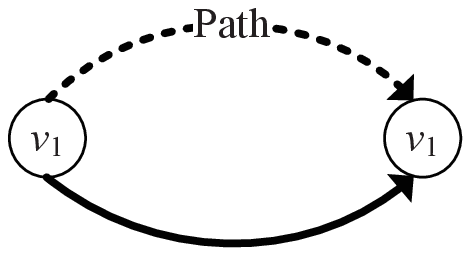}
    \label{fig:merge_redundancy}
  }
  \caption{Conflict and redundancy appearing in merge operation.}
  \label{fig:merge}
\end{figure}

To summarize, the merging operation works as follows. Firstly, all directed edges from both solutions are simply added into a single edge set. Secondly, edges are ranked according to the associated significance measure, calculated by the basic causal solver employed by SADA. Thirdly, a sequential conflict testings are run over the ordered edges non-decreasingly on the significance. An edge is removed if it is conflicted with any of the previous edges. Finally, the redundancy edges are discovered and removed based on results of the conditional independence tests. A complete description is available in Algorithm \ref{algo:merge}.

\begin{algorithm}[!h]
\caption{\textbf{Merge Results}}\label{algo:merge}
\begin{algorithmic}
   \STATE Input: {$G_1$, $G_2$: solutions to $V_1\cup C$ and $V_2\cup C$}\\
   \STATE Output:{$G$: solution for $C \cup V_1\cup V_2$}\\
   \STATE //basic merging
   \STATE $G=G_1 \cup G_2$;\\
   \STATE //conflict removal
   \STATE Sort edges in $G$ in descending order of significance;\\
   \STATE Mark all variable pairs as unreachable;
   \FOR {each $v_1\rightarrow v_2 \in G$}
   {
       \IF {$<v_2, v_1>$ is reachable}
       {
           \STATE $G=G - \{v_1\rightarrow v_2\}$;\\
       }
       \ELSE
       {
           \STATE Mark $\left(v_1, v_2\right)$ as reachable;
       }\ENDIF
   }\ENDFOR
   \STATE //redundancy removal
   \FOR{each $v_1\rightarrow v_2 \in G$}
   {
       \IF {$v_1\rightarrow \cdots \rightarrow v_2$ is in $G$}
       {
           \STATE Let $Path\left(v_1\rightarrow v_2\right)$ includes all variables involved in $v_1\rightarrow \cdots \rightarrow v_2$;\\
           \IF {$\exists V' \subset Path\left(v_1\rightarrow v_2\right)$ satisfies $v_1 \bot v_2 | V'$}
           {
               \STATE $G = G- \{v_1 \rightarrow v_2\}$;\\
           }\ENDIF
       }\ENDIF
   }\ENDFOR
   \RETURN  $G$;\\
\end{algorithmic}
\normalsize
\end{algorithm}

\section{Analysis under Reliable Conditional Independence Test}\label{sec:theory1}

In this section, we study the theoretical properties of SADA, especially on the effectiveness on problem scale reduction and consistency on causal results, under the assumption that no error is introduced by any conditional independence tests. The assumption is mathematically formulated as follows.

\begin{assumption}\label{assum:simple}
For any variables $v_1$, $v_2$ and variable set $V$, the conditional independence tests always return true, \emph{iff} $v_1\bot v_2|V$.
\end{assumption}

Intuitively, when the above assumption holds, the causal cut finding operator never generates wrong partitions which divides a causal-effect variable pair into two separate sets $V_1$ and $V_2$. Although such requirement is unlikely to meet in practice, it simplifies the model and allows us to derive accurate analysis on SADA in the rest of the section.

\subsection{Effectiveness on Scale Reduction}

In this part of the section, we aim to verify the effectiveness of the causal cut finding algorithm. In particular, we try to prove that the scale of the subproblem is significantly reduced by applying the randomized causal cut finding algorithm.

%

\begin{theorem}\label{theorem:convergence}
If every variable has no more than $d_m$ parental variables in CBN, by setting $k=\left(2d_m+2\right)^2$, Algorithm \ref{algo:split} returns a causal cut $\left(C,V_1,V_2\right)$ with probability at least 0.5, such that
$$\min\{|V_1|,|V_2|\}\geq \frac{|V|}{2d_m+2}$$
\end{theorem}
\begin{proof} 
		Since the causal structure must be a DAG, there is at least one topological order on the variables. Here topological order of a DAG is a linear order of its vertices such that for every directed edge $v_i\rightarrow v_j$, $v_i$ comes before $v_j$ in the order. Let $V=\{v_1,v_2,\ldots,v_{|V|}\}$ be the topological order, then $v_i$'s parental variables are ahead of $v_i$ in the order. When randomly picking up variable pairs in $V$, i.e., $u$ and $v$ from $V$, we will first show that $u$ and $v$ generate a causal cut with $\min\{|V_1|,|V_2|\}\geq\frac{|V|}{2d_m+2}$ with probability at least $1/\left(2d_m+2\right)^2$.
		
		Without loss of generality, we assume $n=|V|$ and the variable $u$ is behind $v$ in the topological order over $V$. With probability $\eta$, $u$ is one of the variables between $v_{0.5n}$ and $v_{\left(0.5+\eta\right)n}$. Consider all the $\eta n$ variables between $v_{0.5n}$ and $v_{\left(0.5+\eta\right)n}$. We simply put all these variables in $V_1$, and put all parental variables of $V_1$, denoted by $P\left(V_1\right)$, and all variables behind $v_{\left(0.5+\eta\right)n}$ into $C$. The rest of the variables are inserted into $V_2$. In the configuration $\left(C,V_1,V_2\right)$,  $C$ contains all the parental variable of $V_1$ by adding $P\left(V_1\right)$ into $C$, and all the possible children variables of $V_1$ by adding all variables behind $v_{\left(0.5+\eta\right)n}$ into $C$, because the children variable must be ordered behind $V_1$. Thus, there is no direct edge between $V_1$ and $V_2$, and the configuration $\left(C,V_1,V_2\right)$ is a causal cut.
		
		In the above causal cut $\left(C,V_1,V_2\right)$, $|V_1|=\eta n$ and $|V_2|\geq \frac{n}{2}-\eta  n d_m$. The inequality is because $V_2=\{v_i| i\leq 0.5n \}-P(V_1)$ and  $|P(V_1)| \leq \eta  n d_m$. By picking $\eta=\frac{1}{2d_m+2}$, $\min\{|V_1|,|V_2|\}\geq \frac{n}{2d_m+2}$. When $v$ is selected in $V_2$, Algorithm \ref{algo:split} must converge to a solution better than the artificial configuration above. Because in the above analysis $v$ could be place into $V_2$ or $C$, when $v$ is put into $V_2$, the size of $V_2$ is larger than the above expectation. This happens with probability at least $\frac{1}{\left(2d_m+2\right)^2}$ when $\eta=\frac{1}{2d_m+2}$.
		
		By running the randomized causal cut finding algorithm $k=\left(2d_m+2\right)^2$ times, the probability of finding a causal cut with $\min\{|V_1|,|V_2|\}\geq \frac{n}{2d_m+2}$ is larger than $ 1- \left(1 - \frac{1}{\left(2d_m+2\right)^2} \right) ^ {\left(2d_m+2\right)^2}$.  Since $\left(1 - \frac{1}{\left(2d_m+2\right)^2} \right) ^ {\left(2d_m+2\right)^2} \approx e^{-1}$ when $\left(2d_m+2\right)^2$ is sufficiently large, the probability of finding a causal cut with $\min\{|V_1|,|V_2|\}\geq \frac{n}{2d_m+2}$ is at least $1-e^{-1}$, i.e., larger than 1/2.
\end{proof}

The last theorem implies that the causal cut finding algorithm is effective on reducing the scale of the subproblems. Another implication is on the selection of the parameter $\theta$. To guarantee there is a reduction on problem size, the parameter $\theta$ should be no smaller than $2d_m+2$, since such $\theta$ ensuring that $\frac{\theta}{2d_m+2}\geq 1$.

\subsection{Recall and Precision on Result Causal Edges}

The accuracy of the causation discovery is measured based on the recall and precision on the result causal edges, i.e., the percentage of accurate causal edges and the percentage of causal edges returned. In this section, we show that SADA always finds fully accurate results in terms of recall and precision, if the invoked basic causal solver and conditional independence tests are both reliable. 

\begin{assumption}
A basic causal solver A is reliable, if A always outputs accurate causal edges on any variable set V even with latent confounders.
\end{assumption}

\begin{theorem} \label{theorem:correctcomplete}
Assume $D$ is a set of samples generated from the causal structure $G$ over the variable set $V$. If the basic causal solver $A$ and conditional independence tests used in SADA are both reliable, SADA always finds the true causal structure $G$.
\end{theorem}

\begin{proof}  Assume $G'$ is the causal structure discovered by SADA. We only need to prove the correctness and completeness of $G'$. The correctness and completeness are equivalent to $\forall v_1\rightarrow v_2 \in G'$, $v_1 \rightarrow v_2 \in G$, and $\forall v_1\rightarrow v_2 \in G$, $v_1 \rightarrow v_2 \in G'$, respectively. The details of the proof are given as follows:
	
	\textbf{Completeness:} Assume $v_1\rightarrow v_2 \in G$, firstly, according to the causal cut finding step,  both $v_1$ and $v_2$ must be in one subproblem, $V_1\cup C$ or $V_2 \cup C$, but not acrose the two subproblems. Otherwise, $v_1$ and $v_2$ is conditional independent of each other given some subset of $C$, conflicts with the condition $v_1 \rightarrow v_2 \in G$ and the assumption that the conditional independence tests are reliable. Secondly, according to the following two conditions: '$v_1$ and $v_2$ are in the same subproblem' and 'basic causal solver is reliable', $v_1 \rightarrow v_2 \in G'$ will be discovered in one of the subproblems. Finally, the edge $v_1 \rightarrow v_2$ will not be removed in the merging step. If the edge is removed by either conflict or redundancy reason, it will conflict with the condition $v_1\rightarrow v_2 \in G$ and the assumption that the condition independence test is reliable. Thus, $v_1\rightarrow v_2$ must be contained in the result of SADA, in anther word, $v_1\rightarrow v_2 \in G'$.
	
	\textbf{Correctness:} Assume $v_1\rightarrow v_2 \in G'$,  firstly we will show $v_1\rightarrow v_2$ is the correct result of the subproblem. According to the framework of SADA, $v_1$ and $v_2$ must be discovered in one of the subproblem $V_1 \cup C$ and $V_2 \cup C$.  Without loss of generality, assume $v_1\rightarrow v_2$ is discovered in the subproblem $V_1 \cup C$ by the basic causal solver. According to the condition that the basic causal solver is reliable, $v_1 \rightarrow v_2$ must be the correct result of the subproblem $V_1 \cup C$. Secondly, we will show $v_1\rightarrow v_2 \in G$. If $v_1\rightarrow v_2$ is the correct result of $V_1 \cup C$ but not contained in $G$, then there must exist a variable set $V' \subset V$ satisfies $v_1 \bot v_2 | V'$. Thus, there must be a path $v_1 \rightarrow  \cdots \rightarrow v_2$ which contains $V'$ as intermediate nodes. If such path exists, according to the merging step, $v_1\rightarrow v_2$ will be removed from the result set $G'$, and conflicts with the condition that $v_1\rightarrow v_2 \in G'$. Thus, $v_1\rightarrow v_2 \in G$.
\end{proof}

Basically, the theorem above claims that the recall and precision on the causal edges returned by SADA are always satisfiable. However, as we emphasized at the beginning of the section, the assumption on reliable conditional independence tests is impractical, since randomness and noises always exist in the samples. In next section, we relax the assumption and show that SADA remains effective in a class of much more general settings.

\section{Analysis under Vulnerable Conditional Independence Test}\label{sec:theory2}

\begin{figure}[ht]
\centering
\includegraphics[width=\columnwidth]{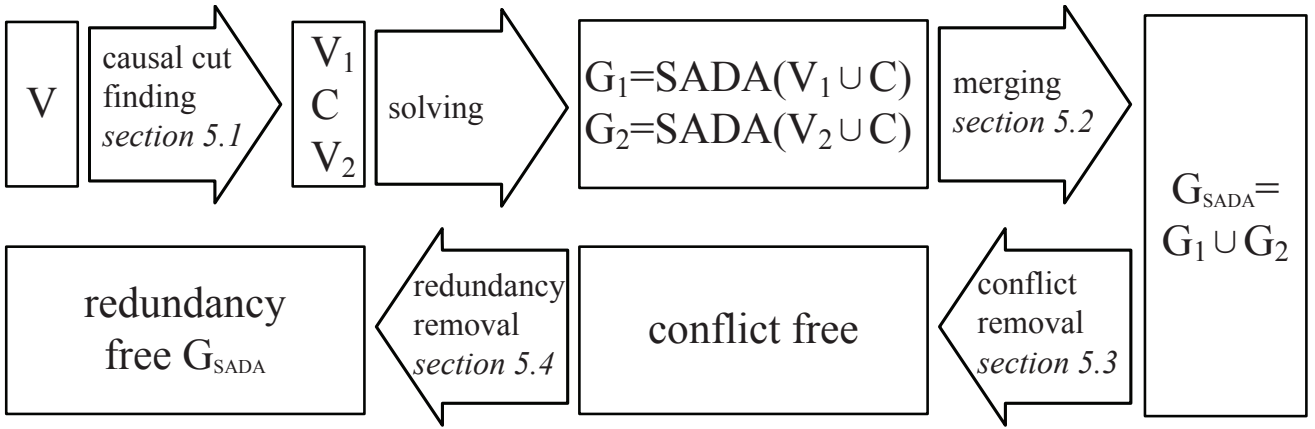}
\caption{Workflow of SADA with single partitioning.} \label{fig:sada_framework}
\end{figure}

In this section, we analyze the performance of SADA under more general assumptions, taking errors incurred by causal cut finding, merging, redundancy removal and conflict removal into consideration. To accomplish the goals of the analysis, we investigate the impact of each step in the algorithm on the recall and precision of the results one at a time. As is shown in Figure \ref{fig:sada_framework}, there are five key steps taken in SADA, including (1) finding causal cut $\left(C, V_1, V_2\right)$ over variable set $V$; (2) solving two subproblem $V_1\cup C$ and $V_2\cup C$ using basic causal solvers independently; (3) merging two sub-solutions $G_1$ and $G_2$; (4) removing conflict edges after the merging; and (5) detecting and removing redundancy edges.

To improve the readability of the paper, we summarize all the notations used in the rest of the section in Table \ref{tab:notation}. Basically, $n$s indicate sizes of different subgraphs, $e$s indicate numbers of actual causal edges in the subgraphs, $f$s denote the number of non-causal ordered pairs, $G$s represent resulting causal structures of the problems, $R$s and $P$s are recalls and precisions on the causal edges in the results, and $r$s are the probabilities of returning a particular false causal edge in the results. We will also utilize the following equations, $e+f=n^2-n$, $\frac{eR}{eR+fr}=P$, $e_c+f_c=\left(n_c\right)^2-n_c$, which are trivial extensions of the definitions.

\begin{table*}[t]
\center
\caption{Table of Notations}
\label{tab:notation}
\begin{tabular}{|c|p{4.5in}|}
\hline Notations & Description  \\
\hline $n, n_1, n_2, n_c$ & \# of variables in $V$, $V_1\cup C$, $V_2\cup C$ and $C$\\
\hline $e, e_1, e_2, e_c$ & \# of causal edges in $V\!\times\! V$, $\left(V_1\!\cup\! C\right)\! \times\! \left(V_1\!\cup\!C\right)\! - \!C \!\times\! C $, $\left(V_2\!\cup\!C\right) \!\times \!\left(V_2\!\cup\!C\right) \!- \!C \!\times\! C $ and $C \!\times \!C$\\
\hline $f, f_1, f_2, f_c$ & \# of non-causal ordered pairs in $V\!\times\! V$, $\left(V_1\!\cup\! C\right)\! \times\! \left(V_1\!\cup\!C\right)\! - \!C \!\times\! C $, $\left(V_2\!\cup\!C\right) \!\times \!\left(V_2\!\cup\!C\right) \!- \!C \!\times\! C $ and $C \!\times \!C$\\
\hline $d$ & average in-degree of the causal structure\\
\hline $G, G_1, G_2$ & solution of $V$, $V_1\cup C$ and $V_2\cup C$\\
\hline $R, R_1, R_2$ & recall of $G$, $G_1$ and $G_2$\\
\hline $P, P_1, P_2$ & precision of $G$, $G_1$ and $G_2$\\
\hline $P_m, P_{co}, P_{re}$ & precision after merging, conflict removal and redundancy removal\\
\hline $r, r_1, r_2$ & falsely discovered probability of non-cause edges in $G$, $G_1$ and $G_2$\\
\hline $e_{m},f_{m}$ & \# of causal edges and non-causal ordered pairs discovered in basic merging step\\
\hline $e_{ce},e_{co},e_{re}$ & \# of causal edges falsely removed in causal cut finding, conflict and redundancy removal\\
\hline $\Lambda(V_1,V_2)$ & \# the event that there is no causal edge across $V_1$ and $V_2$ detected by the conditional independence tests\\
\hline $\delta$ & the largest positive constant such that $R_1 \geq R+\delta$ and $R_2 \geq R+\delta$ always hold\\
\hline $\gamma$ & the largest positive constant such that $r \leq r_1-\gamma$ and $r \leq r_1-\gamma$ always hold\\
\hline $\alpha$ & the error probability of conditional independence tests on returning $\left(v_1,v_2,V\right)$ that $v_1\not\bot v_2|V$ \\
\hline $\beta$  & the error probability of conditional independence tests on \emph{not} returning $\left(v_1,v_2,V\right)$ that $v_1\bot v_2|V$ \\
\hline $\varepsilon$ & the probability of a falsely discovered causal edge has higher significance than a true causal edge\\
\hline
\end{tabular}
\normalsize
\end{table*}


The analysis is derived based on the following general assumptions, which are commonly satisfied in real world settings. They are motivated by our observations on SADA in empirical evaluations. To begin with, the first assumption addresses the property of the causal structure.

\begin{assumption}\label{assum:edgedistr}
In the causal structure, the edges are uniformly distributed on the nodes. 
\end{assumption}

The above assumption ensures the local causal structures are independent of each other, which is reasonable in most of the application scenarios. Under this assumption, each edge $v_i \rightarrow v_j$ appears with probability $c/\left(n-1\right)$ when the in-degree of the variable $v_j$ is $c$.

The following assumption attempts to build the connection between recall/precision of the origin problem and recall/precision of the sub-problems.

\begin{assumption}\label{assum:randp}
There exist global constants $\delta>0$ and $\gamma>0$, such that $R_1 \geq R+\delta$, $R_2 \geq R+\delta$, $r_1 \leq r-\gamma$ and $r_2 \leq r-\gamma$.
\end{assumption}

The above assumption is based on the observation that the scale of $V_1 \cup C$ and $V_2 \cup C$ is usually significantly smaller than that of $V$. Given the fixed sample set used for training, it is common to gain accuracy improvement when the basic causal solver is run on problems of smaller scale. This assumption is also empirically validated in the experiments with results available in Figure \ref{fig:sim_split}.

Next assumption is used to model the significance of the discovered edges, which is crucial to the analysis of conflict and redundancy removal.

\begin{assumption}\label{assum:sig}
Given a true discovered edge $v_1 \rightarrow v_2$ and a falsely discovered edge $v_3 \rightarrow v_4$, there exists a global constant $\varepsilon >0$, such that $\Pr\left(sig\left(v_1 \rightarrow v_2\right) > sig\left(v_3 \rightarrow v_4\right)\right)>1-\varepsilon$.
\end{assumption}

As defined in the partial result merging algorithm (in section \ref{sec:merge}), the significance measure is the $p$-value of the edge's reliability. It is thus reasonable that the correctly discovered causal edges are more likely to get higher significance than the falsely discovered edges.


Finally, the last assumption regards the reliability of the conditional independence tests, on two types of errors in the results of conditional independence tests.

\begin{assumption}\label{assum:cit}
In the conditional independence tests, the probability that the independent relation is correctly identified as independent is at least $1-\alpha$, and
the probability that the dependent relation is falsely identified as independent is at most $\beta$.
\end{assumption}

In practice, the error bounds $\alpha$ and $\beta$ could be tuned by the users by specifying appropriate confidence interval. In the rest of the paper, without other specification, we use $0.05$ as the default values for $\alpha$ and $\beta$.


\subsection{Effects of Causal Cut}

Causal cut benefits SADA algorithm by improving the recall and precision of the basic causal solver when applied on sub-problems with much smaller scale. The side effect of the causal cut is the additional error overhead caused by undetected causal variable pairs which are separated in the causal cut. In this part of the section, we aim to give an analysis on the expectation of the causal cutting error in the partitioning step of SADA. To simplify the analysis, instead of using the original randomized causal cut finding algorithm, we uniformly divides the variables into $V_1$, $C$ and $V_2$, given the specific sizes $n_1$, $n_2$, $n_c$. Note that such partitioning is unaware of the actual causal structure. The causal cutting error incurred by SADA algorithm is thus definitely smaller than the estimation.

We assume that $V_1$, $C$ and $V_2$ are random variable sets output by the uniform assignment. Based on the assumptions, it is equivalent to assign the causal edges into the graph, from a null causal structure on the fixed partitioning result $V_1$, $C$ and $V_2$. We thus derive all the probabilities by simulating the random edge assignment process as following.

Let $\Psi$ denote the set of all causal structures over the current variable set $V$, and $\Psi_i$ denote a subset of $\Psi$ with exactly $i$ edges between $V_1$ and $V_2$. Given Assumption \ref{assum:edgedistr}, the probability of having an actual structure $\psi\in\Psi_i$ could be evaluated using Equation (\ref{eqn:cut1}), as the edges are independently assigned to the variables in $V$ and there are $e$ actual causal edges and $f$ non-causal ordered pairs,

\begin{equation}\label{eqn:cut1}
\Pr\left(\psi\in\Psi_i\right) = {i \choose n_1n_2} \left(\frac{e}{f+e}\right)^i \left(\frac{f}{f+e}\right)^{n_1n_2-i}
\end{equation}

Intuitively, in the equation, $\frac{e}{f+e}$ denotes the probability that there is a direct edge between a particular pair of variables. Similarly, $\frac{f}{f+e}$ denotes the probability that there is no edge between a particular pair of variables.

Moreover, we can further evaluate the probability of generating a valid partitioning in terms of the algorithmic condition in SADA. Particularly, SADA does not accept a partitioning if it finds a potential causal edge between any $v_1\in V_1$ and $v_2\in V_2$. We thus derive Equation (\ref{eqn:cut2}) below to evaluate the joint probability of $\psi \in \Psi_i$ and $\Lambda(V_1,V_2)$. $\Lambda(V_1,V_2)$ refers to the event that there is no causal edge across $V_1$ and $V_2$ detected by the conditional independence tests:

\begin{equation}\label{eqn:cut2}
\Pr\left(\psi \in \Psi_i, \Lambda(V_1,V_2)\right) = {i \choose n_1n_2} \left(\frac{e}{f+e}\right)^i \left(\frac{f}{f+e}\right)^{n_1n_2-i}\beta^i\left(1-\alpha\right)^{n_1n_2-i}
\end{equation}


Given Equation (\ref{eqn:cut1}) and Equation (\ref{eqn:cut2}), we apply Bayesian rule to estimate the probability of generating a partition $\left( C, V_1, V_2\right)$ by the causal cut finding algorithm, under the condition of $\Lambda(V_1,V_2)$, i.e.,

\begin{eqnarray}\label{eqn:cut3}
\Pr\left(\psi \in \Psi_i|\Lambda(V_1,V_2)\right) & = & \frac{P\left(\psi \in \Psi_i, \Lambda(V_1,V_2)\right)}{P\left(\Lambda(V_1,V_2)\right)} \nonumber \\
& = & \frac{{i \choose n_1n_2} \left(\frac{e}{f+e}\right)^i \left(\frac{f}{f+e}\right)^{n_1n_2-i}\beta^i \left(1-\alpha\right)^{n_1n_2-i}} {\sum_{j=0}^{n_1n_2} {j \choose n_1n_2} \left(\frac{e}{f+e}\right)^j \left(\frac{f}{f+e}\right)^{n_1n_2-j}\beta^j \left(1-\alpha\right)^{n_1n_2-j}}\nonumber \\
& = & \frac{{i \choose n_1n_2} \left(\frac{e  \beta}{f \left(1-\alpha\right)}\right)^i} {\sum_{j=0}^{n_1n_2} {j \choose n_1n_2} \left(\frac{e\beta}{f\left(1-\alpha\right)}\right)^j}
\end{eqnarray}

The expectation of the error $e_{ce}$ caused by the causal cut finding, i.e., the number of undetected edges across $V_1$ and $V_2$, could be calculated as

\begin{eqnarray}\label{eqn:cut_bound}
&   & \mbox{Exp}\left(e_{ce}\right) \nonumber\\
& = & \sum_{i=0}^{n_1n_2} i\Pr\left(\psi \in \Psi_i|\Lambda(V_1,V_2)\right)\nonumber\\
& = & \frac{\sum_{i=0}^{n_1n_2} i  {i \choose n_1n_2} \left(\frac{e\beta}{f\left(1-\alpha\right)}\right)^i} {\sum_{j=0}^{n_1n_2} {j \choose n_1n_2} \left(\frac{e\beta}{f\left(1-\alpha\right)}\right)^j} \nonumber\\
& = & \frac{\phi\sum_{i=0}^{n_1n_2}{i \choose n_1n_2}  \left(\frac{e\beta}{f\left(1-\alpha\right)}\right)^i+\sum_{i=0}^{n_1n_2} \left(i-\phi\right)  {i \choose n_1n_2} \left(\frac{e\beta}{f\left(1-\alpha\right)}\right)^i} {\sum_{j=0}^{n_1n_2} {j \choose n_1n_2} \left(\frac{e\beta}{f\left(1-\alpha\right)}\right)^j}\nonumber\\
& = & \frac{\phi\sum_{i=0}^{n_1n_2}{i \choose n_1n_2} \left(\frac{e\beta}{f\left(1-\alpha\right)}\right)^i+\sum_{i=0}^{n_1n_2} \left(i-\phi\right) {i-\phi \choose n_1n_2} \left(\frac{e\beta}{f\left(1-\alpha\right)}\right)^{i-\phi}\frac{{i \choose n_1n_2}}{C_{n_1n_2}^{\left(i-\phi\right)}}\left(\frac{e\beta}{f\left(1-\alpha\right)}\right)^\phi} {\sum_{j=0}^{n_1n_2} {j \choose n_1n_2} \left(\frac{e\beta}{f\left(1-\alpha\right)}\right)^j}\nonumber \\
&\leq & \phi+1\nonumber\\
&\leq & \left\lceil\frac{n^2e\beta}{4f\left(1-\alpha\right)}\right\rceil+1
\end{eqnarray}
in which $\phi$ is the smallest positive integer satisfying the condition $ \frac{{i \choose n_1n_2}}{{i-\phi \choose n_1n_2}}\left(\frac{e\beta}{f\left(1-\alpha\right)}\right)^\phi \leq 1$ for any integer $i\in [\phi+1, n_1n_2]$. Because $\frac{{i \choose n_1n_2}}{{i-\phi \choose n_1n_2}}\leq \left(\frac{n_1n_2}{\phi}\right)^\phi$ holds for $ \forall i\in [\phi+1, n]$, $\phi\leq\lceil\frac{n^2e\beta}{4f\left(1-\alpha\right)}\rceil+1$ and $E\left(e_{ce}\right)$ is no larger than $\phi+1$, correspondingly.

The causal cutting error is usually small, since $\frac{e\beta}{f\left(1-\alpha\right)}$ is not large in most cases. Under a typical setting with variable number $n=100$, in-degree $d=1.25$, $\alpha=0.05$ and  $\beta=0.05$,
the expectation of the causal cutting error is no larger than $\left\lceil\frac{n^2e\beta}{4f\left(1-\alpha\right)}\right\rceil+1=3$, which is much smaller than the expected number of causal edges at 125.

\subsection{Effects of Result Merging}

In the merging step of SADA, the algorithm simply includes all the resulting edges from the solutions to the subproblems, i.e., $G_1$ and $G_2$. The key to our analysis in this part is to understand the recall and precision on the causal edges within the variable set $C$, because they are calculated in both subproblems on $V_1\cup C$ and $V_2\cup C$. To make the analysis possible, we try to evaluate the accuracy on these edges in $C\times C$ by estimating the number of true causal edges and false causal edges returned in the merging step.


Since the recalls of $G_1$ and $G_2$ are $R_1$ and $R_2$ respectively, and $G_1$ and $G_2$ are solved independently, the number of actual causal edges identified in $C$ is $e_c(1-\left(1-R_1\right)\left(1-R_2\right))=e_c(R_1+R_2-R_1R_2)$. Similarly, the number of falsely discovered edges $C$ is $f_c(1-\left(1-r_1\right)\left(1-r_2\right))=f_c(r_1+r_2-r_1r_2)$.

Therefore, we could derive the number of true causal edges and false causal edges by the following two equations:

\begin{equation}\label{eqn:merge_true}
e_{m}=e_1R_1+e_2R_2+e_c\left(R_1+R_2-R_1R_2\right)
\end{equation}

and

\begin{equation}\label{eqn:merge_false}
f_{m}=f_1r_1+f_2r_2+f_c\left(r_1+r_2-r_1r_2\right)
\end{equation}

Based on Equation (\ref{eqn:merge_true}), we can further derive the lower bound on the number of returned causal edges. Note that the third inequality is due to Assumption \ref{assum:randp}, the fourth inequality is based on the fact $e_1+e_2+e_c+e_{ce}=e$, and the last inequality applies the rule $R+\delta \leq 1$.

\begin{eqnarray}
e_{m} &=& e_1R_1+e_2R_2+e_c\left(R_1+R_2-R_1R_2\right)\nonumber\\
      &\geq& e_1R_1+e_2R_2+e_cR_1 \nonumber\\
      &\geq& \left(e_1+e_2+e_c\right)R+\left(e_1+e_2+e_c\right)\delta+e_c\left(R+\delta\right)\nonumber\\
      &\geq& \left(e-e_{ce}\right)R+\left(e-e_{ce}\right)\delta\nonumber\\
      &\geq& eR+e\delta-e_{ce}
\end{eqnarray}

Thus, the lower bound on the expectation $\mbox{Exp}\left(e_{m}\right)$ could be derived as follows, in which the inequality is based on the upper bound of  $\mbox{Exp}\left(e_{ce}\right)$ available in Equation  (\ref{eqn:cut_bound}).

\begin{equation}\label{eqn:merge_bound}
\begin{split}
\mbox{Exp}\left(e_{m}\right) &\geq eR+e\delta-\mbox{Exp}\left(e_{ce}\right)\\
          &\geq eR+e\delta-\left\lceil\frac{n^2e\beta}{4f\left(1-\alpha\right)}\right\rceil-1
\end{split}
\end{equation}

The following lemma provides a sufficient condition to generate higher precision on the causal edges in $C$ than that of the basic causal solver directly applied on the original problem.

\begin{lemma}\label{lemma:merge_precision}
If $\delta > \frac{Pf_c\left(r-r^2\right)}{\left(1-P\right)\left(e_1+e_2+e_c\right)}$ or $\gamma > \frac{f_cr}{f_1+f_2+2f_c}$, $P_{m}\geq P$ holds.
\end{lemma}

\begin{proof}
	Basically, Equation (\ref{eqn:merge_bound}) implies that the precision is higher, i.e., $P_{m}\geq P$, if $\delta > \frac{Pf_c\left(r-r^2\right)}{\left(1-P\right)\left(e_1+e_2+e_c\right)}$. When the condition is satisfied, we have
	
	\begin{eqnarray}\label{eqn:merge_precision1}\small
	P_{m} &=& \frac{e_m}{e_m+f_m}\nonumber\\
	&=& \frac{e_1R_1+e_2R_2+e_c\left(R_1+R_2-R_1R_2\right)} {e_1R_1+e_2R_2+e_c\left(R_1+R_2-R_1R_2\right)+ f_1r_1+f_2r_2+f_c\left(r_1+r_2-r_1r_2\right)}\nonumber\\
	&\geq& \frac{e_1R_1+e_2R_2+e_c\left(R_1+R_2-R_1R_2\right)} {e_1R_1+e_2R_2+e_c\left(R_1+R_2-R_1R_2\right)+ f_1r+f_2r+2f_cr}\nonumber\\
	&\geq& \frac{\left(e_1+e_2+e_c\right)R+\left(e_1+e_2+e_c\right)\delta+e_c\left(R+\delta-\left(R+\delta\right)^2\right)} {\left(e_1\!+\!e_2\!+\!e_c\right)R+\left(f_1\!+\!f_2\!+\!f_c\right)r+\left(e_1\!+\!e_2\!+\!e_c\right)\delta+f_c\left(r\!-\!r^2\right)\!+\!e_c\left(R\!+\!\delta\!-\!\left(R+\delta\right)^2\right)}\nonumber\\
	&\geq& \frac{\left(e_1+e_2+e_c\right)R+\left(e_1+e_2+e_c\right)\delta} {\left(e_1+e_2+e_c\right)R+\left(f_1+f_2+f_c\right)r+\left(e_1+e_2+e_c\right)\delta+f_c\left(r-r^2\right)}\nonumber\\
	&\geq& P\nonumber
	\end{eqnarray}
	
	The first equality is based on the definition of precision. The first inequality is because of the facts $r_1\leq r$ and $r_2\leq r$ given in Assumption \ref{assum:randp}. The second inequality is derived based on $R_1\geq R+\delta$ and $R_2\geq R+\delta$ given in Assumption \ref{assum:randp}. And the last inequality is due to $\frac{\left(e_1+e_2+e_c\right)R} {\left(e_1+e_2+e_c\right)R+\left(f_1+f_2+f_c\right)r}=P$ and $\delta \geq \frac{pf_c\left(r-r^2\right)}{\left(1-P\right)\left(e_1+e_2+e_c\right)}$.
	
	Similarly, when $\gamma \geq \frac{f_cr}{f_1+f_2+2f_c}$, we can derive the bounds on $P_{m}$ by another way as:
	\begin{eqnarray}\label{eqn:merge_precision2}
	& & P_{m}\nonumber\\
	&=& \frac{e_1R_1+e_2R_2+e_c\left(R_1+R_2-R_1R_2\right)} {e_1R_1+e_2R_2+e_c\left(R_1+R_2-R_1R_2\right)+f_1r_1+f_2r_2+f_c\left(r_1+r_2-r_1r_2\right)}\nonumber\\
	&\geq& \frac{e_1R+e_2R+e_c\left(R+R-R^2\right)} {e_1R+e_2R+e_c\left(R+R-R^2\right)+f_1r_1+f_2r_2+f_c\left(r_1+r_2-r_1r_2\right)}\nonumber\\
	&\geq& \frac{\left(e_1+e_2+e_c\right)R} {\left(e_1+e_2+e_c\right)R+\left(f_1+f_2+f_c\right)r+f_cr-\left(f_1+f_2+2f_c\right)\gamma}\nonumber\\
	&\geq& P\nonumber
	\end{eqnarray}
	
	The first equality is based on the definition of precision. The first inequality is because of $R_1\geq R$ and $R_2\geq R$, given in Assumption \ref{assum:randp}. The second inequality is because of $r_1\leq r-\gamma$ and $r_2\geq r-\gamma$, given in Assumption \ref{assum:randp}.
	The last inequality is because of $\frac{\left(e_1+e_2+e_c\right)R} {\left(e_1+e_2+e_c\right)R+\left(f_1+f_2+f_c\right)r}=P$ and $\gamma \geq \frac{f_cr}{f_1+f_2+2f_c}$. This completes the proof of the lemma.
\end{proof}

\subsection{Effects of Conflict Removal}

The step of conflict removal is expected to eliminate the false causal edges returned by the merging step, under the potential risk of falsely removing actual causal edges. As is shown in Algorithm \ref{algo:merge}, the selection of the removal edges heavily depends on the significance measure employed on candidate edges. In this part of the section, we analyze how the randomness on the significance measure affects the accuracy of results after conflict removal.

Given an edge $v_i \rightarrow v_j$, there are two types of conflicts to address, including (1) conflict between two edges, e.g. $v_i \rightarrow v_j$ against  $v_i \leftarrow v_j$; and (2) conflict between an edge and a path, e.g. $v_i \rightarrow v_j$ against $v_i \ldots \leftarrow \ldots v_j$.

In the first type of conflict, the variable pair $v_i$ and $v_j$ exist on both $G_1$ and $G_2$. Thus, the number of conflict edge pairs between $v_i \rightarrow v_j$ and $v_i \leftarrow v_j$ can be estimated as $e_c\left(r_2R_1+R_2r_1\right)+f_cr_1r_2$. In the estimation, $e_cr_2R_1$ denotes the number of actual causal edges correctly discovered in $G_1$ with a corresponding reversed edge included in $G_2$.  Similarly, $e_cr_1R_2$ denotes the number of actual causal edges correctly discovered in $G_2$, while a reversed one is available in $G_1$ at the same time. Finally, $f_cr_1r_2$ is the number of edge pairs, which are both false and reversed to each other. Based on Assumption \ref{assum:sig}, the expected number of actual edges removed by the current step is $\varepsilon e_c\left(r_2R_1+R_2r_1\right)$, by only considering pairs with at least one actual causal edge. 

The second type of conflict is in the form $v_i \rightarrow v_j$ and $v_i \ldots \leftarrow \dots v_j$.
Because the solutions to the subproblems, i.e., $G_1$ and $G_2$, are acyclic and there is no direct edge across the variable set $V_1$ and $V_2$, the conflict of second type are definitely triggered by the edges within $C\times C$. When merging results from $G_1$ and $G_2$ in terms of the edges in $C\times C$, there are $e_cR_2\left(1-R_1\right)$ additional true causal edges and $\left(n_c^2-n_c-e_c\right)r_2\left(1-r_1\right)$ additional false causal edges incurred by $G_2$.
Similarly, there are $e_cR_2\left(1-R_1\right)+f_cr_2\left(1-r_1\right)$ extra edges from the results $G_1$ when merging the $G_2$'s results on $C$ to $G_1$.
Thus, there are $e_c\left(R_1+R_2-2R_1R_2\right)+f_cr_2\left(r_1+r_2-2r_1r_2\right)$ edges potentially triggering conflicts between the edges.

Consider a particular edge $v_i \rightarrow v_j$ and the counter-result with path $v_i \ldots \leftarrow \dots v_j$. If there are $k$ intermediate variables on the path, the path appears with probability at most ${n-2\choose k}\left(\frac{d}{n}\right)^{k+1}$, in which $d$ is the maximal in-degree in the variables. By iterating on all possible lengths from 1 to $n-2$, the expected number of conflicted paths triggered by $v_i \rightarrow v_j$ is at most $\sum_{k=1}^{n-2} {n-2\choose k}\left(\frac{d}{n}\right)^{k+1}=\frac{d}{n}\left(\left(1+\frac{d}{n}\right)^{n-2}-1\right)= \frac{d}{n}\left(1+\frac{d}{n}\right)^{n-2}-\frac{d}{n}$.

During the conflict removal step, it is necessary and sufficient to remove exactly one edge with the lowest significance on the path to break the conflict. Such a removed edge is either an actual causal edge with the lowest significance or a false causal edge with the lowest significance. According to Assumption \ref{assum:sig}, the probability of generating lower significance for an actual causal edge against a false causal edge is as small as $\varepsilon$. It facilitates us to calculate an upper bound on the removed actual causal edges by $\left(e_c\left(R_1+R_2-2R_1R_2\right)+f_c\left(r_1+r_2-2r_1r_2\right)\right)\varepsilon$.

Combing both two types of conflicts, the expected number of actual edges removed in the conflict removal step over all conflict cases are upper bounded in Equation (\ref{eqn:conflict_bound}).
The first inequality is derived by the fact that $R_1$, $R_2$, $r_1$ and $r_2$ are no greater than $1$. The second inequality is because $e_c+f_c=\left(n_c\right)^2-n_c < \left(n_c\right)^2$.

\begin{eqnarray}\label{eqn:conflict_bound}
& & \mbox{Exp}\left(e_{co}\right)\nonumber\\
&=& \varepsilon \left(e_c\left(r_2R_1\!+\!R_2r_1\right)+ \left(\frac{d}{n}\left(1\!+\!\frac{d}{n}\right)^{n-2}\!\!-\!\frac{d}{n}\right)\left(e_c\left(R_1\!+\!R_2\!-\!2R_1R_2\right)+f_c\left(r_1\!+\!r_2\!-\!2r_1r_2\right)\right)\right)\nonumber\\
              &\leq& \varepsilon \left(2e_cr+ \frac{d}{n}\left(1+\frac{d}{n}\right)^{n-2}e_c+2f_cr\right)\nonumber\\
              &\leq& \varepsilon \left(2\left(n_c\right)^2r+ \frac{d^2n_c}{n}\left(1+\frac{d}{n}\right)^{n-2}\right)
\end{eqnarray}

The following lemma gives the sufficient condition to ensure that the precision never drops after the conflict removal step in SADA.

\begin{lemma}\label{lemma:conflict_precision}
When the error of significance measure $\varepsilon$ is no larger than $1-P$, the precision on the returned causal edges after conflict removal never drops.
\end{lemma}

\begin{proof}
	To prove the lemma, we take each conflict into consideration and update the precision on the results once at a time.
	
	
	Due to the acyclic property of the actual causal structure, each conflict cycle must contain at least one false causal edge. We consider two different type of cases in this proof. The first type includes cases of conflict containing false causal edges only. Since at least one false causal edge is removed in the step, the conflict removal definitely improves the precision.
	
	%
	%
	
	The second type of conflicts contains at least one actual causal edge in each conflict. With probability no larger than $\epsilon$, an actual causal edge is removed from the result, otherwise a false causal edge is deleted. Assume that there are $e'$ actual causal edges and $f'$ false causal edges in the result at this particular moment. When $\epsilon\leq 1-P$, the expectation of the new precision after breaking this conflict is no smaller than $P$.
	
	\begin{equation}\label{eqn:conflict_precision2}
	\begin{split}
	\varepsilon \frac{e'-1}{e'+f'-1} + \left(1-\varepsilon\right) \frac{e'}{e'+f'-1}= \frac{e'-\varepsilon}{e'+f'-1}\geq P.
	\end{split}
	\end{equation}
	
	The last inequality is derived by $\frac{e'}{e'+f'}\geq P$ and $\varepsilon \leq 1-P$. This completes the proof of the lemma.
\end{proof}

Given the conclusion of the lemma, when $\varepsilon$ is sufficiently smaller, the conflict removal always brings benefit to the precision of the results, i.e., $\mbox{Exp}\left(P_{co}\right)\geq P$ holds.

\subsection{Effects of Redundancy Removal}

We apply similar analysis strategy on the redundancy removal step as is done on the conflict removal step. Since the redundancy between $v_i\rightarrow v_j$ from $G_1$ and $v_i\rightarrow v_j$ from $G_2$ is already broken in the basic merging step, we only need to consider the redundancy between $v_i\rightarrow v_j$ and $v_i \rightarrow \cdots \rightarrow v_j$ in this step of SADA algorithm.

Similar to the results on the conflict removal step, the path $v_i \dots \rightarrow \dots v_j$ with $k$ intermediate variables appears with probability ${{n-2} \choose k}\left(\frac{d}{n}\right)^{k+1}$.
Considering all the paths with length within the range $k\in[1, n-2]$, the expected number of redundancy path for $v_i\rightarrow v_j$  is $\sum_{k=1}^{n-2} {n-2\choose k}\left(\frac{d}{n}\right)^{k+1}=\frac{d}{n}\left(\left(1+\frac{d}{n}\right)^{n-2}-1\right)= \frac{d}{n}\left(1+\frac{d}{n}\right)^{n-2}-\frac{d}{n}$.
Moreover, there are $e_c\left(R_1+R_2-2R_1R_2\right)+f_c\left(r_1+r_2-2r_1r_2\right)$ extra edges potentially triggering redundancy cycles.
To eliminate the redundancy for each of the case, it is to remove $\beta$ actual causal edges in average, because the conditional independence tests are used to detect such redundancy. The expected number of actual causal edges removed in the redundancy step over all redundancy cases is thus upper bounded by the following formula, with limited impact on the recall of the results.

\begin{equation}\label{eqn:redundancy_bound}
\begin{split}
\mbox{Exp}\left(e_{re}\right)& \leq \beta \left(e_c\left(R_1+R_2-2R_1R_2\right)+f_c\left(r_1+r_2-2r_1r_2\right)\right)\\
              &\leq \beta \left(\frac{d}{n}\left(1+\frac{d}{n}\right)^{n-2}e_c+2\left(n_c\right)^2r\right)\\
              &\leq \beta \left(2\left(n_c\right)^2r+ \frac{d^2n_c}{n}\left(1+\frac{d}{n}\right)^{n-2}\right)\\
\end{split}
\end{equation}

Regarding the precision, the following lemma gives the sufficient condition to the improvement on precision by the redundancy removal step.

\begin{lemma}\label{lemma:redundancy_precision}
If the error probability $\beta$ is no larger than $1-P$, the expected precision $\mbox{Exp}\left(P_{re}\right)\geq P$ holds.
\end{lemma}

\begin{proof}
	Let $P'=\frac{e'}{e'+f'}$ denote the current precision $P'=\frac{e'}{e'+f'}$ and $P''$ is the precision after removing one potential redundant edge. Given the condition $P'\geq P$,
	We have the following inequality about $P''$.
	
	\begin{equation}\label{eqn:redundancy_precision}
	\begin{split}
	\mbox{Exp}\left(P''\right)&=\beta \frac{e'-1}{e'-1+f'} + \left(1- \beta\right) \frac{e'}{e'+f'-1}= \frac{e'-\beta}{e'+f'-1}\geq P
	\end{split}
	\end{equation}
	The last inequality is because of $\frac{e'}{e'+f'}=P' \geq P$ and $\beta \leq 1-P$.
	
	In the potential redundancy removal, the initial precision is $P'=P_{co}$ and $P'\geq P$ holds. Thus, $\mbox{Exp}\left(P_{re}\right)\geq P$ holds.
\end{proof}

\subsection{Overall Evaluation of SADA}

In this section, we combine all the results in previous subsection and provide an overall evaluation on the recall/precision of SADA.

Equation (\ref{eqn:recall}) gives an estimation on the recall after each recursive partitioning step. Theorem \ref{theorem:recall} gives the sufficient condition of SADA's recall is higher than that of the basic causal solver.

\begin{equation}\label{eqn:recall}
    \mbox{Exp}\left(R_{SADA}\right)=\frac{\mbox{Exp}\left(e_{m})-\mbox{Exp}\left(e_{co}\right)-\mbox{Exp}\left(e_{re}\right)\right)}{e}\\
\end{equation}

\begin{theorem}\label{theorem:recall}
When $\delta \geq e^{-1}\left(\lceil\frac{n^2e\beta}{4f\left(1-\alpha\right)}\rceil+1+\left(\varepsilon+\beta\right)n_c\left(2n_cr+ \frac{d^2}{n}\left(1+\frac{d}{n}\right)^{n-2}\right)\right)$, we always have $\mbox{Exp}\left(R_{SADA}\left(G\right)\right) \geq R\left(G\right)$.
\end{theorem}
\begin{proof}
	Combing the lower bound on the number of discovered edges in the merging step (in Equation (\ref{eqn:merge_bound})), the upper bound of falsely removed causal edges in the conflict removal (in Equation (\ref{eqn:conflict_bound}))
	and redundancy removal step (in Equation (\ref{eqn:redundancy_bound})), we come to the following conclusion on the expectation of $R_{SADA}$:
	
	\begin{equation}\label{eqn:recall2}
	\begin{split}
	\mbox{Exp}\left(R_{SADA}\right)&=\frac{\mbox{Exp}\left(e_{m}\right)-\mbox{Exp}\left(e_{co}\right)-\mbox{Exp}\left(e_{re}\right)}{e}\\
	&\geq \frac{eR+e\delta-\lceil\frac{n^2e\beta}{4f\left(1-\alpha\right)}\rceil-1 -\left(\varepsilon+\beta\right)n_c\left(2n_cr+ \frac{d^2}{n}\left(1+\frac{d}{n}\right)^{n-2}\right) }{e} \\
	&\geq R
	\end{split}
	\end{equation}
	
	The last inequality is because of $\delta \geq \frac{1}{e}\left(\lceil\frac{n^2e\beta}{4f\left(1-\alpha\right)}\rceil+1+\left(\varepsilon+\beta\right)n_c\left(2n_cr+ \frac{d^2}{n}\left(1+\frac{d}{n}\right)^{n-2}\right)\right)$.
\end{proof}

The above theorem gives the sufficient condition to generate higher recall than the basic causal solver directly applied on the original problem. In the following, we demonstrate that this sufficient condition can be easily satisfied in real applications. Given the typical setting with variable number $n=100$, average in-degree $d=1.25$, $\alpha=0.05$,  $\beta=0.05$, $\varepsilon=0.05$ and $n_c=10$, the minimal $\delta$ required is $0.0404$. We also illustrate the minimal $\delta$s under different average in-degree $c$ and causal cut set size $n_c$ in Figure \ref{fig:delta}. The results in the figure shows that $5\%$ improvement on subproblems with smaller domain is enough to help improve the accuracy by employing SADA. Even when the size of causal cut set is as large as 20, SADA is capable of achieving better accuracy if the basic causal solver is able to improve $10\%$ on the subproblems. Figure \ref{fig:delta_nc} also reflects the fact that it is important to control the size of causal cut set, under which SADA could guarantee more performance enhancements.

\begin{figure} [h]
  \centering
  \subfigure[Degree]{
    \includegraphics[width=0.47\textwidth]{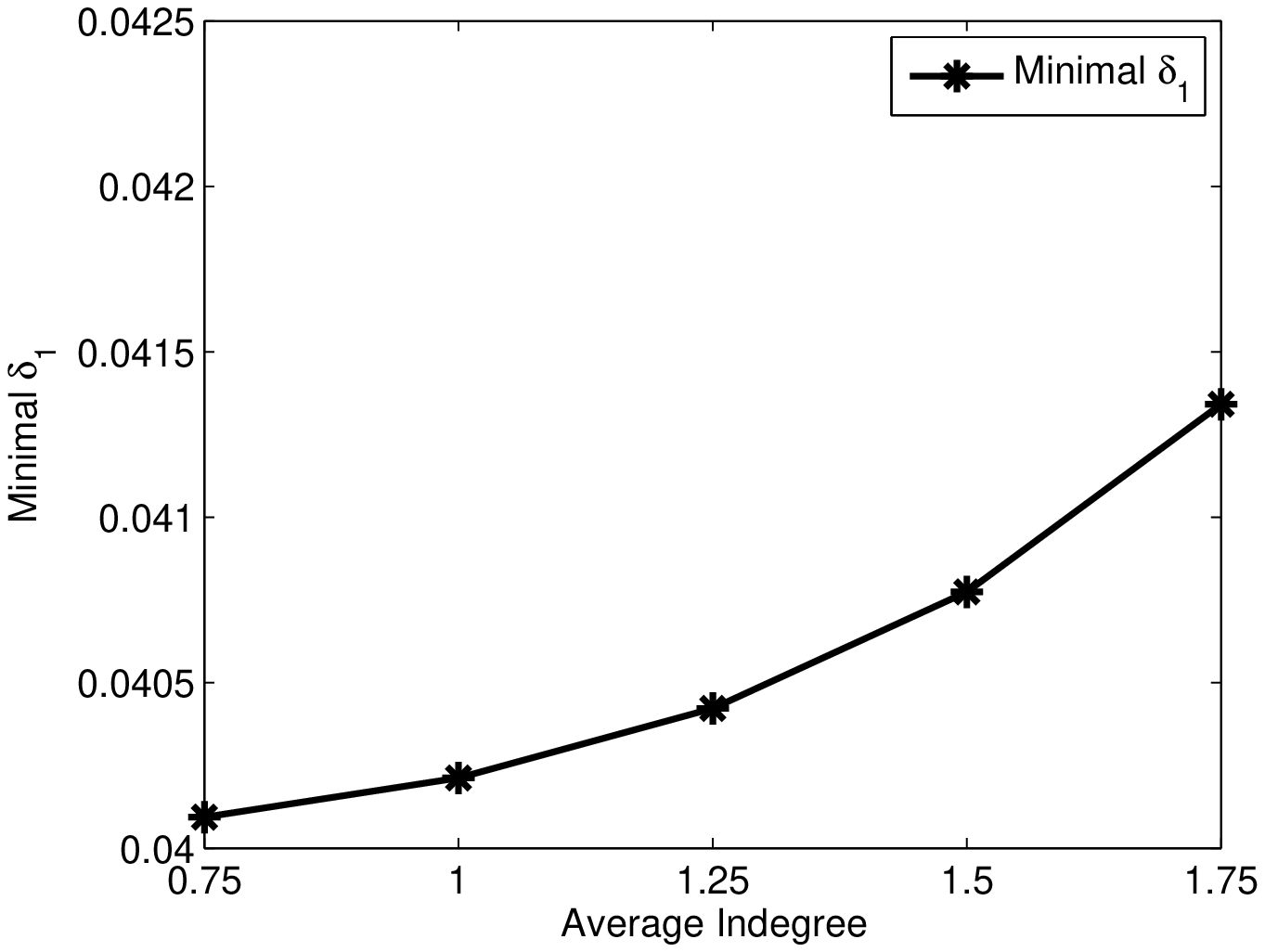}
    \label{fig:delta_degree}
  }
  \subfigure[Size of Causal Cut Set]{
    \includegraphics[width=0.47\textwidth]{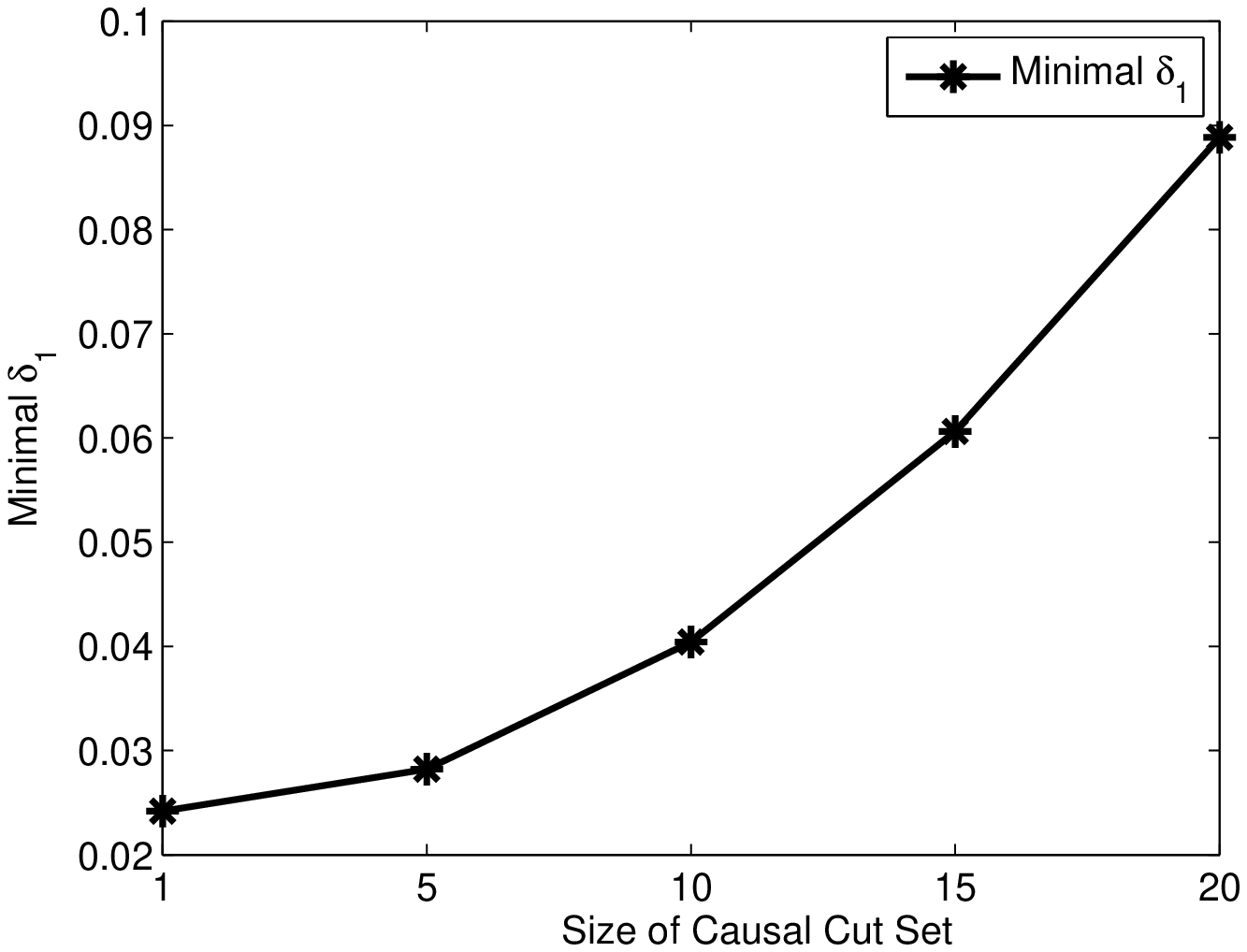}
    \label{fig:delta_nc}
  }
  \caption{Minimal $\delta$ with different average in-degree and size of causal cut set.}
  \label{fig:delta}
\end{figure}

\begin{theorem}\label{theorem:precision}
When the following three conditions hold: (1) $\delta > \frac{Pf_c\left(r-r^2\right)}{\left(1-P\right)\left(e_1+e_2+e_c\right)}$ or $\gamma > \frac{f_cr}{f_1+f_2+2f_c}$, (2) $\varepsilon$ is no larger than $1-P$ and (3) $\beta$ is no larger than $1-P$, we always have $\mbox{Exp}\left(P_{SADA}\right)\geq P$.
\end{theorem}
\begin{proof}
	Based on Lemma \ref{lemma:merge_precision}, $P_{m}\geq P$ holds after the merging step.
	According to Lemma \ref{lemma:conflict_precision} and \ref{lemma:redundancy_precision}, the precision is not also reduced in each process of conflict removal and redundancy removal. Thus, $\mbox{Exp}\left(P_{SADA}\right)\geq P$ holds.
\end{proof}

Given the same setting as used in Figure \ref{fig:delta}, when the basic causal solver achieves precision $P=0.5$ on the original problem, SADA improves the precision when $\delta >0.08$ or $\gamma >0.002$. Although the minimal requirement on $\delta$ for better precision is higher than that for better recall, both the increase on the recall on true causal edges and the decrease on the number of falsely discovered edges could contribute to the improvement of the precision.


Note that both the conditions given in Theorem \ref{theorem:recall} and Theorem \ref{theorem:precision} are \emph{only} sufficient conditions to the accuracy improvements in SADA, and with a number of loose inequalities are in the proof of the theorems. Thus, SADA improves the performance of the basic causal solvers under much more general conditions in practice, i.e., SADA still improves the performance of the basic causal solvers even when the above conditions are not fully satisfied. In the experiments, we empirically evaluate the effects and verify the advantages of SADA.

\section{Experiments}\label{sec:exp}

\subsection{Experiment Settings}

We evaluate our proposal on datasets generated by simulated and different real-world Bayesian network structures\footnote{\url{www.cs.huji.ac.il/site/labs/compbio/Repository/}}, under linear non-Gaussian model and discrete additive noise model. Because of the non-existence of large scale causal inference problem with ground truth, simulated data on the given structures is used in most of causal structure learning methods \cite{aliferis_local_2010,kalisch2007pc}. Please note that only faithfulness condition and causal sufficiency assumption are employed in the generic SADA framework. Additional compatible data generation assumptions, linear non-Gaussian assumption and additive noise assumption, are employed for linear non-Gaussian model and discrete additive noise model, respectively.

\noindent\textbf{Linear Non-Gaussian Acyclic Model}\\
Under the assumption of linear non-Gaussian acyclic model, the samples are generated based on linear functions as $v_i=\sum_{v_j\in P\left(v_i\right)}w_{ij}v_j+e_i$. When randomly generating these linear functions, we restrict that $\sum_{P\left(v_i\right)}w_{ij}=1$ and the variance $Var\left(e_i\right)=1$ for every variable $v_i$. 

On the linear non-Gaussian acyclic model, our algorithm is compared with LiNGAM \cite{ShimizuJMLR06}, DLiNGAM \cite{ShimizuJMLR2011Direct} and Sparse-ICA LiNGAM \cite{zhang2009sica}. The implementation of LiNGAM and DLiNGAM are from the authors of the paper. The implementation of Sparse-ICA LiNGAM is based on the sparse-ICA of \cite{zhang2009sica}, and the pruning algorithm of \cite{ShimizuJMLR06}.  For SADA, we employ the conditional independence tests following the method proposed in \cite{Baba2004PartialCorrelation}, with threshold at 95\%. LiNGAM \cite{ShimizuJMLR06} with Wald test is appointed as the basic casual solver $A$ after SADA reaches the minimal scale threshold $\theta$ at subproblems. On all datasets, SADA stops the partitioning when the subproblem reaches the size $\theta=10$. The recursive partitioning is also terminated when Algorithm \ref{algo:split} fails to find any causal cut. LiNGAM without applying any division is also used as the baseline approach, when reporting recall, precision and F1 score. Note that, when the variable size is larger than 100, LiNGAM cannot perform Wald test due to memory consumption constraint (i.e., one Wald test cannot be finished on a sever with 64GB memory for a whole day).

\noindent\textbf{Discrete Additive Noise Model}\\
The generation process of the discrete data follows the cyclic method used in \cite{Janzing2011tpamidiscrete} under additive noise model (ANM) for causation discovery on discrete data. Each variable is restricted to 3 different values and values are randomly generated based on conditional probability tables. The implementation of SADA for discrete domain is slightly different from that for continuous domain. $G^2$ test \cite{spirtes2011} is employed as the conditional independence test, with the threshold at $95\%$. The basic causal solver $A$ called by SADA is a brute force method to find all causalities on problems of small scale. Again, the brute-force ANM without variable division is also employed as a baseline approach. The ANM algorithm checks every possible pair of variables following the method proposed in \cite{Janzing2011tpamidiscrete}, and the redundancy and conflict edges are removed using the similar method as the merging step of SADA.

In all the experiments, the evaluation metric includes, \emph{causal cutting error}, \emph{recall}, \emph{precision} and \emph{F1 score}. The causal cutting error ratio is $e_{ce}/e$, i.e., the number of falsely removed causal edges in the causal cut finding step to the number of all causal edges. F1 score is calculated as $\frac{2P\times R}{P+R}$, in which $R$ and $P$ are recall and precision on the causal edge results respectively. Causal cutting error is evaluated on SADA, and the other three metric are evaluated for both SADA and the baseline method. The experiments are compiled and run with Matlab 2009a on a windows PC equipped with a dual-core 2.93GHz CPU and 2GB RAM, and a Linux sever with a 16-core 2.0GHz CPU and 60GB RAM. All Matlab codes of the causation discovery and the generator for linear non-Gaussian data are available at \url{https://sites.google.com/site/cairuichu/SADA.zip} for testing.

\subsection{Results on Simulated Structure}

The main purpose of this group of experiments is to evaluate the scale effect on the basic causal solvers and the sensitivity of our proposal to the variable size, sample size, connectivity, and other characteristics of the causal structures.

The simulated structures are randomly generated under control of a few parameters, including the variable size and average in-degree. The average in-degree reflects the local connectivity of the causal structure. In the structure generation process, all variables are sorted in topological order of the simulated causal structure, so that parent variables are always ahead of children variables. The samples are generated exactly in the order, ensuring that the values of parent variables are generated before the children nodes. The details of the causal structure generation can be found in Algorithm \ref{algo:csg} in the appendix. Given the causal structures, the data is generated using linear non-Gaussian model or discrete model as described earlier. For linear non-Gaussian model, an additional parameter, called noise weight, is used to control the ratio of noise in the data generation process, please refer to Algorithm \ref{algo:ldg} for the details. The parameters settings, including both structure generation and data generation phrases, are given in Table \ref{tab:para}, with default values highlighted in bold font.

\begin{table}[!ht]
\center
\caption{Ranges and Defaults of the Parameters in Simulated Structures}
\label{tab:para}
\begin{tabular}{|c|c|}
\hline Paremeter & Range \\
\hline Variable Size & \{25, 50, \textbf{100}, 200, 400 \}\\
\hline Sample Size &\{50, 100, \textbf{200}, 400, 800 \} \\
\hline In-degree & \{0.75, 1, \textbf{1.25}, 1.5, 1.75 \} \\
\hline Noise Weight in LiNGAM & \{0.1, 0.2, \textbf{0.3}, 0.4, 0.5 \} \\
\hline
\end{tabular}
\end{table}

\noindent\textbf{Effects of Partitioning on the Basic Causal Solver's Recall and Precision}

In the theoretical analysis of the framework, we have shown that the improvement of recall and precision highly depends on the amount of accuracy gain on subproblems with smaller scales. While the conditions theoretically guarantee effectiveness, the actual impact is hardly reflected in the fairly loose bounds. To understand the actual effects of the partitioning, we test the splitting procedure under careful control on the sample size, in-degree and noise ratio, and vary on the number of variables.


The effects of partitioning on the basic causal solver's recall and precision are summarized in Figure \ref{fig:sim_split}, on the subproblems with size\{100,50,25,13,6\}. On the linear non-Gaussian data, the precision increases after each split, and the benefits on recall emerge when the variable size is small enough. Since SADA is capable of generating smaller subproblems for the basic causal solvers, the split brings benefits to both recall and precision. On discrete data, the recall is relative stable for it only checks each pair of variables independently and has less connection to the domain size, i.e., the number of variables. The precision increases after each split, as the suspicions causal pairs can be removed in the split step.
The phenomena reflect that the two sufficient condition for the improvement of recall and precision can be more easily satisfied on the linear non-Gaussian data, while only the sufficient condition for the improvement of precision can be easily satisfied on the discrete data.

\begin{figure} [h]
  \centering
  \subfigure[Linear Non-Gaussian Model]{
    \includegraphics[width=0.47\textwidth]{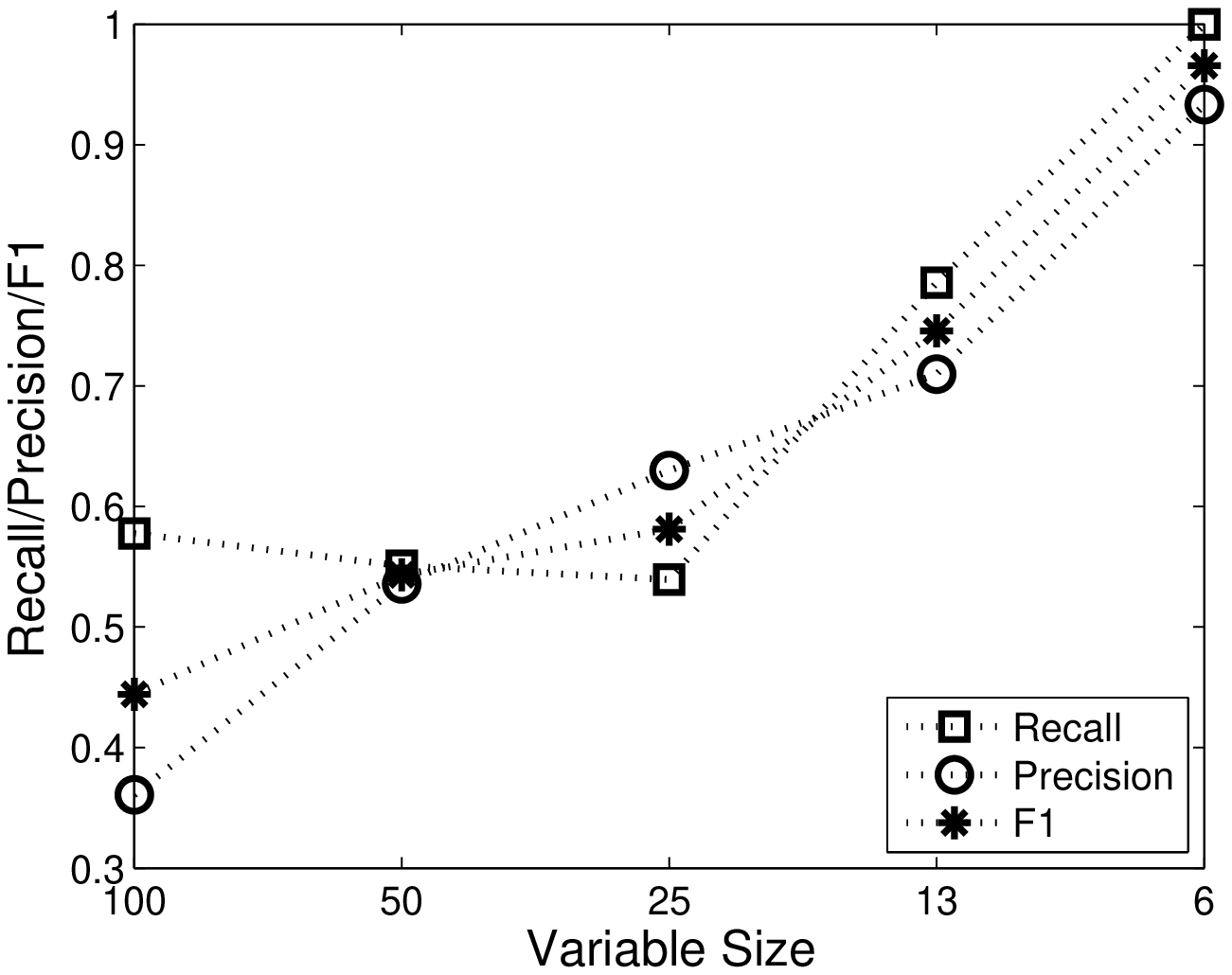}
    \label{fig:sim_split_lingam}
  }
  \subfigure[Discrete Model]{
    \includegraphics[width=0.47\textwidth]{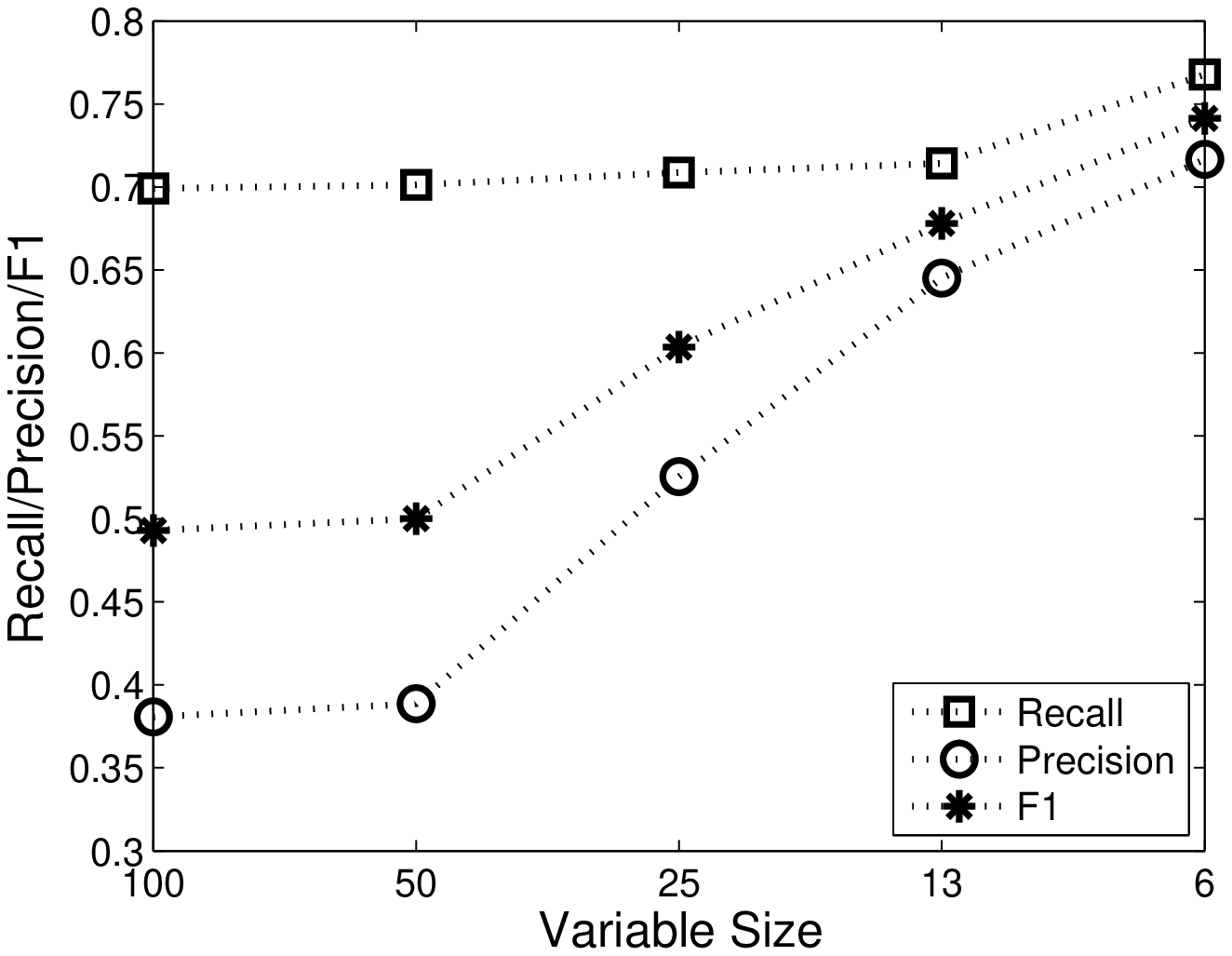}
    \label{fig:sim_split_disc}
  }
  \caption{Split effect in recall and precision of basic causal solver.}
  \label{fig:sim_split}
\end{figure}

\noindent\textbf{Sensitivity to Repeat Time of Finding Causal Cut}

	The repeat time of finding causal cut is another parameter of SADA. In Theorem \ref{theorem:convergence}, we provide a bound on the causal cut size when $k=(2d_m+2)^2$. It is interesting to investigate the effect of the different setting of this parameter. In this experiment, we following parameters are tested,  $k=\{1, 5, 10, 15, 20\}$, where 20 is setting based on the $k=(2d_m+2)^2$. 
	
	Figure \ref{fig:sim_sens_repeat} shows the sensitivity of SADA to the repeat time of finding causal cut. It is interesting to find that the algorithm works well when $k=1$, and the improvement is trivial with increasing of the repeat time. Because of the high computational complexity of the causal cut finding, the repeat time is 1 for all the following experiments.

\begin{figure} [h]
	\centering
	\subfigure[Linear non-Gaussian model]{
		\includegraphics[width=0.47\textwidth]{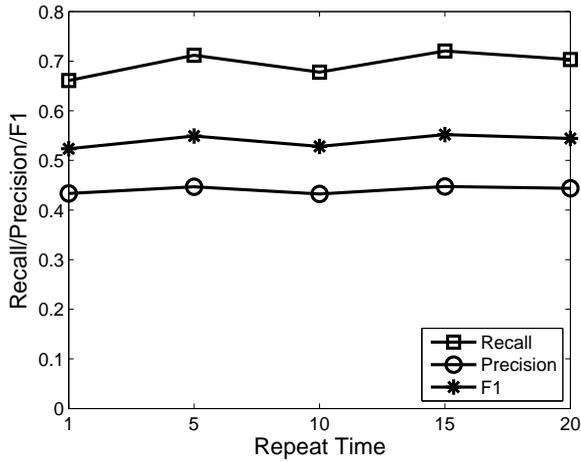}
		\label{fig:sim_sens_repeat_lingam}
	}
	\subfigure[Discrete model]{
		\includegraphics[width=0.47\textwidth]{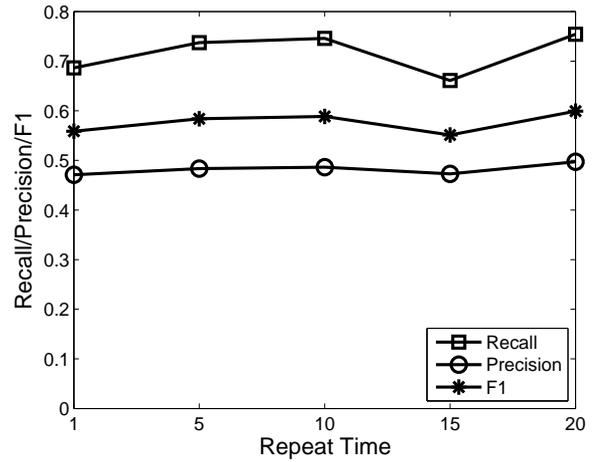}
		\label{fig:sim_sens_repeat_disc}
	}
	\caption{Scalability to the repeat time of finding causal cut.}
	\label{fig:sim_sens_repeat}
\end{figure}

\noindent\textbf{Scalability to Domain Size}

Figure \ref{fig:sim_sens_var_lingam} and Figure \ref{fig:sim_sens_var_disc} report the effects on recall, precision and F1 score, under varying number of variables, on linear non-Gaussian data and discrete data. Generally, SADA works much better on all different numbers of variables. Note that the gap between the methods grows when more variables are in the data domain. This property ensures SADA's scalability to large domains. On the linear non-Gaussian data, LiNGAM fails to work when the variable size is larger than 100, while SADA still achieves good accuracy performance. The figures further strengthens the conclusions of Theorem \ref{theorem:recall} and Theorem \ref{theorem:precision}. In particular, on linear non-Gaussian data, both the sufficient conditions for the improvement of precision and recall are satisfied (as illuminated in Figure \ref{fig:sim_split_lingam}). Compared against the LiNGAM, both the recall and precision are lifted. On discrete data, the sufficient condition for the improvement of precision is satisfied as well (as illuminated in Figure \ref{fig:sim_split_disc}), the improvement mainly stems from the enhanced precision on the results.


\begin{figure} [h]
  \centering
  \subfigure[Linear non-Gaussian model]{
    \includegraphics[width=0.47\textwidth]{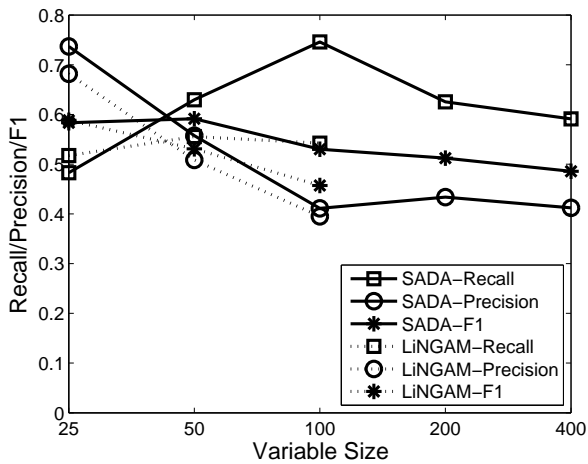}
    \label{fig:sim_sens_var_lingam}
  }
  \subfigure[Discrete model]{
    \includegraphics[width=0.47\textwidth]{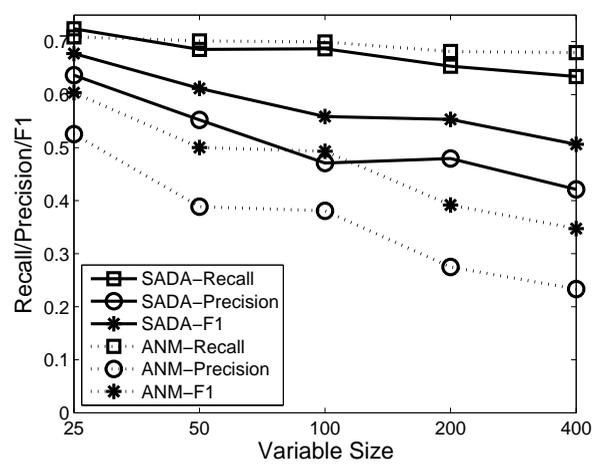}
    \label{fig:sim_sens_var_disc}
  }
  \caption{Scalability to the variable size.}
  \label{fig:sim_sens_var}
\end{figure}

\noindent\textbf{Sensitivity to the Sample Size}

Figure \ref{fig:sim_sens_sample} analyzes the sensitivity of SADA and the compared methods to the sample size.
SADA works better than the compared methods, regardless of the sample size on both linear non-Gaussian data and discrete data.
Moreover, SADA also works well even when LiNGAM fails to work on the linear non-Gaussian data, as show in Figure \ref{fig:sim_sens_sample_lingam}. When sample size is 50 or 100, the sample size is smaller than the number of variables. In such case, LiNGAM fails to work, while SADA framework performs well. This is because SADA can effectively split the original problem into small subproblems solvable to the basic causal solver with small sample size. This is a fundamental advantage of the SADA framework. The performance of SADA improves slightly with the increase on sample size. There are two reasons behind the improvement. Firstly, large sample size improves the reliability of the conditional independence tests (i.e., smaller $\alpha$ and $\beta$) used to find the causal cut and reduce the causal cutting error of SADA. Secondly, large sample size helps basic causal solvers on accuracy, as there are more observations to identify reliable causal relations.


\begin{figure} [h]
  \centering
  \subfigure[Linear non-Gaussian model]{
    \includegraphics[width=0.47\textwidth]{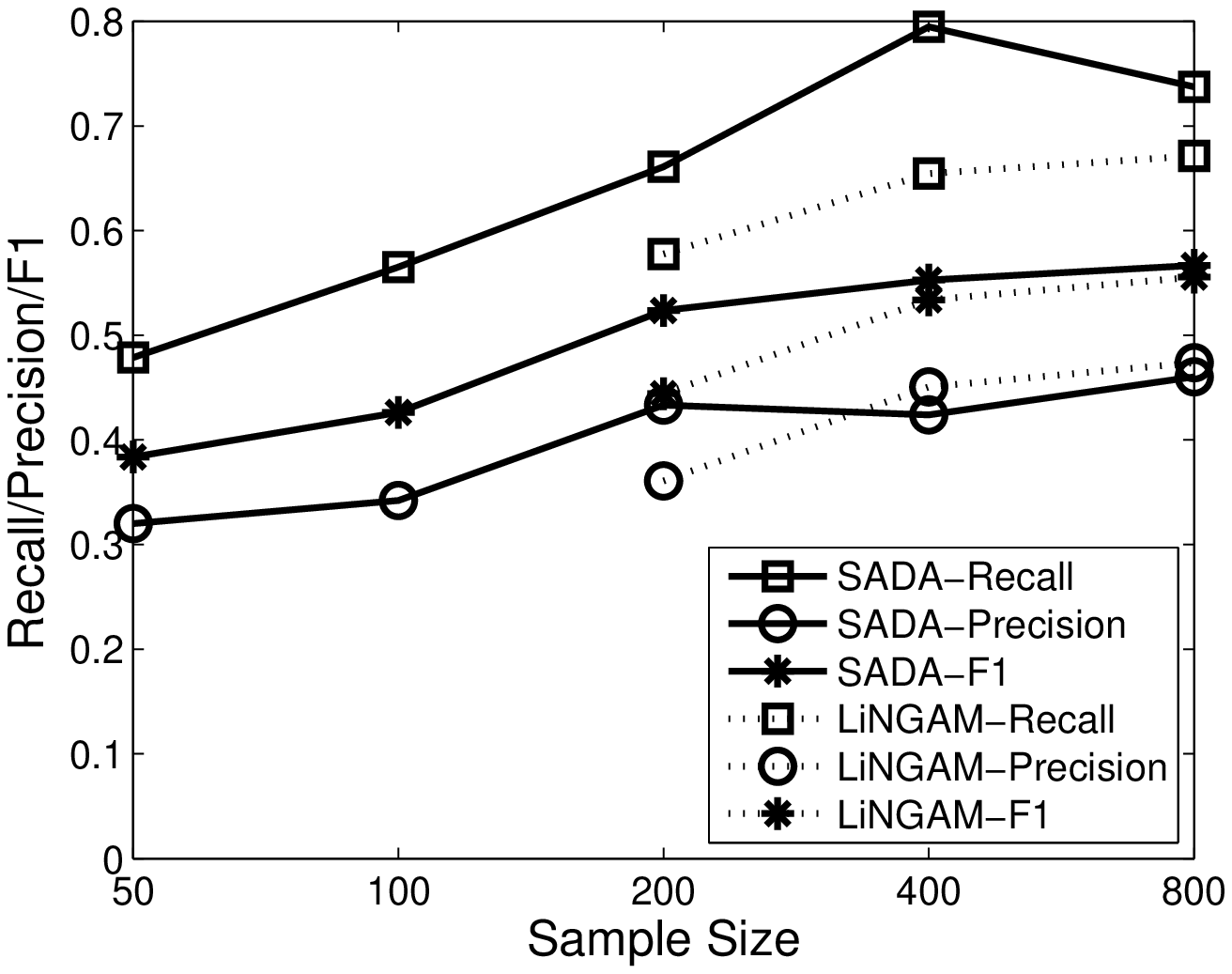}
    \label{fig:sim_sens_sample_lingam}
  }
  \subfigure[Discrete model]{
    \includegraphics[width=0.47\textwidth]{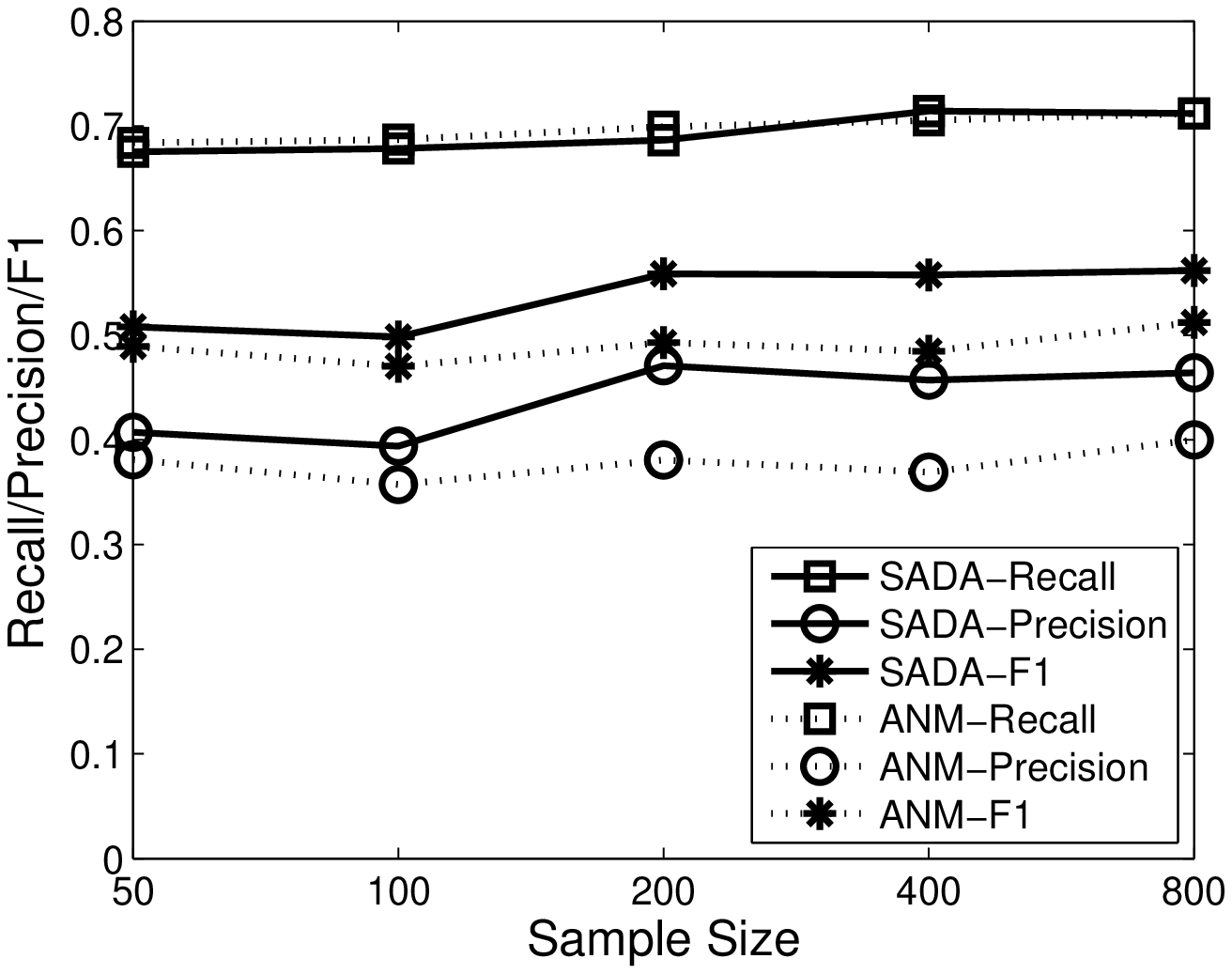}
    \label{fig:sim_sens_sample_disc}
  }
  \caption{Sensitivity to the sample size.}
  \label{fig:sim_sens_sample}
\end{figure}

\noindent\textbf{Sensitivity to the Connectivity}

Figure \ref{fig:sim_con_sample} shows the sensitivity of the algorithms to the average in-degree, an important metric to reflect the connectivity of causal structures.
The performance of SADA drops with the growth of average in-degree, caused mainly by the large causal cut set size in the dense causal structures.
As analyzed in Theorem \ref{theorem:convergence}, the causal cut finding strategy is highly dependent on the in-degree. Thus, the increasing average in-degree will reduce the
quality of the causal cut and increase the causal cut set size. Recall the conclusion in Theorem \ref{theorem:recall} and Theorem \ref{theorem:precision}, the sufficient condition of the improvement will be difficult to be satisfied in the partition with large causal cut set.
%
Though SADA's advantage over LiNGAM is small when the average in-degree is 1.75 on the linear non-Gaussian data, SADA is still competitive, for most of the real world causal structures are sparse as discussed in Table \ref{tab:dataset}. Similar conclusions could be drawn on the discrete data.

\begin{figure} [h]
  \centering
  \subfigure[Linear non-Gaussian model]{
    \includegraphics[width=0.47\textwidth]{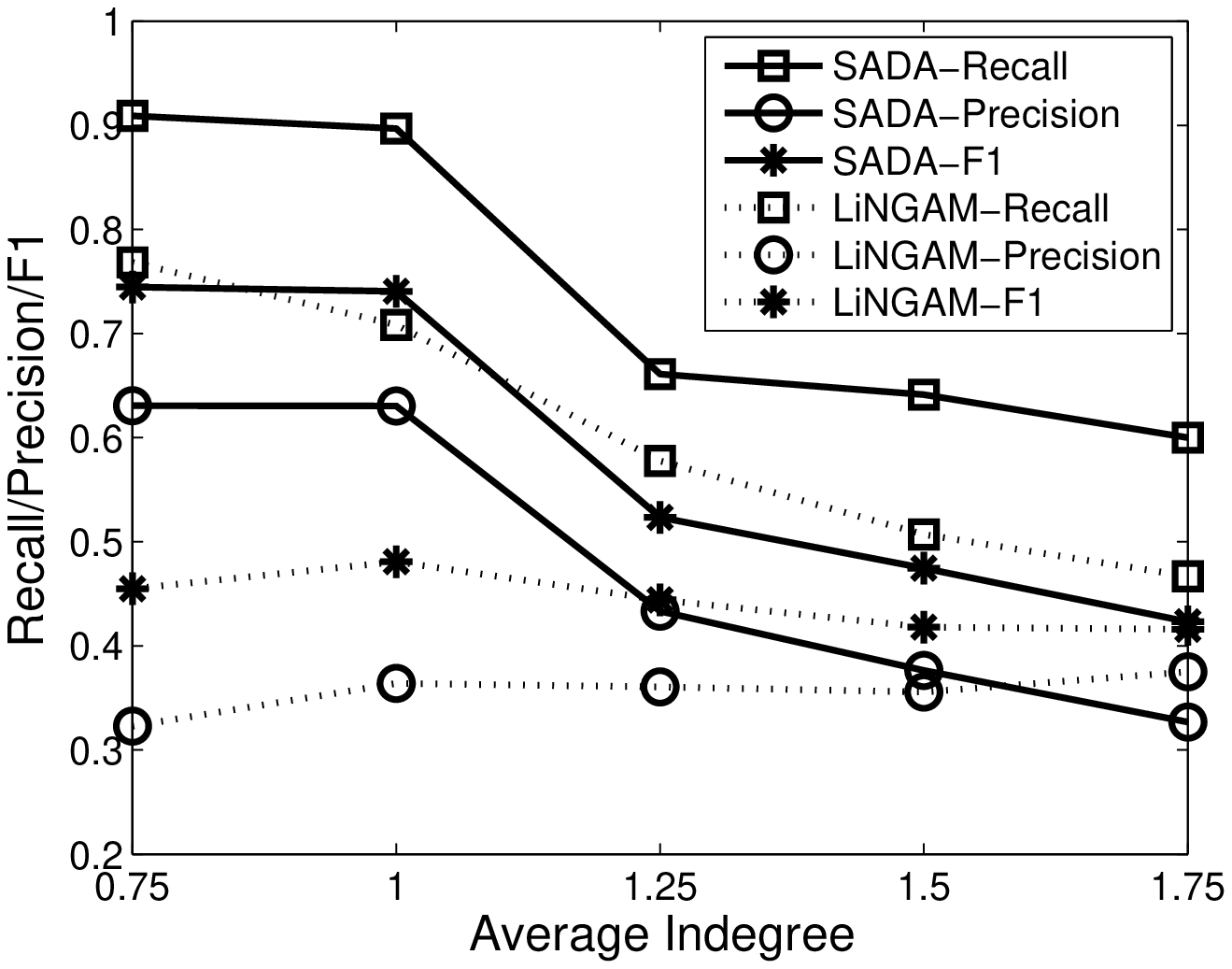}
    \label{fig:sim_sens_con_lingam}
  }
  \subfigure[Discrete model]{
    \includegraphics[width=0.47\textwidth]{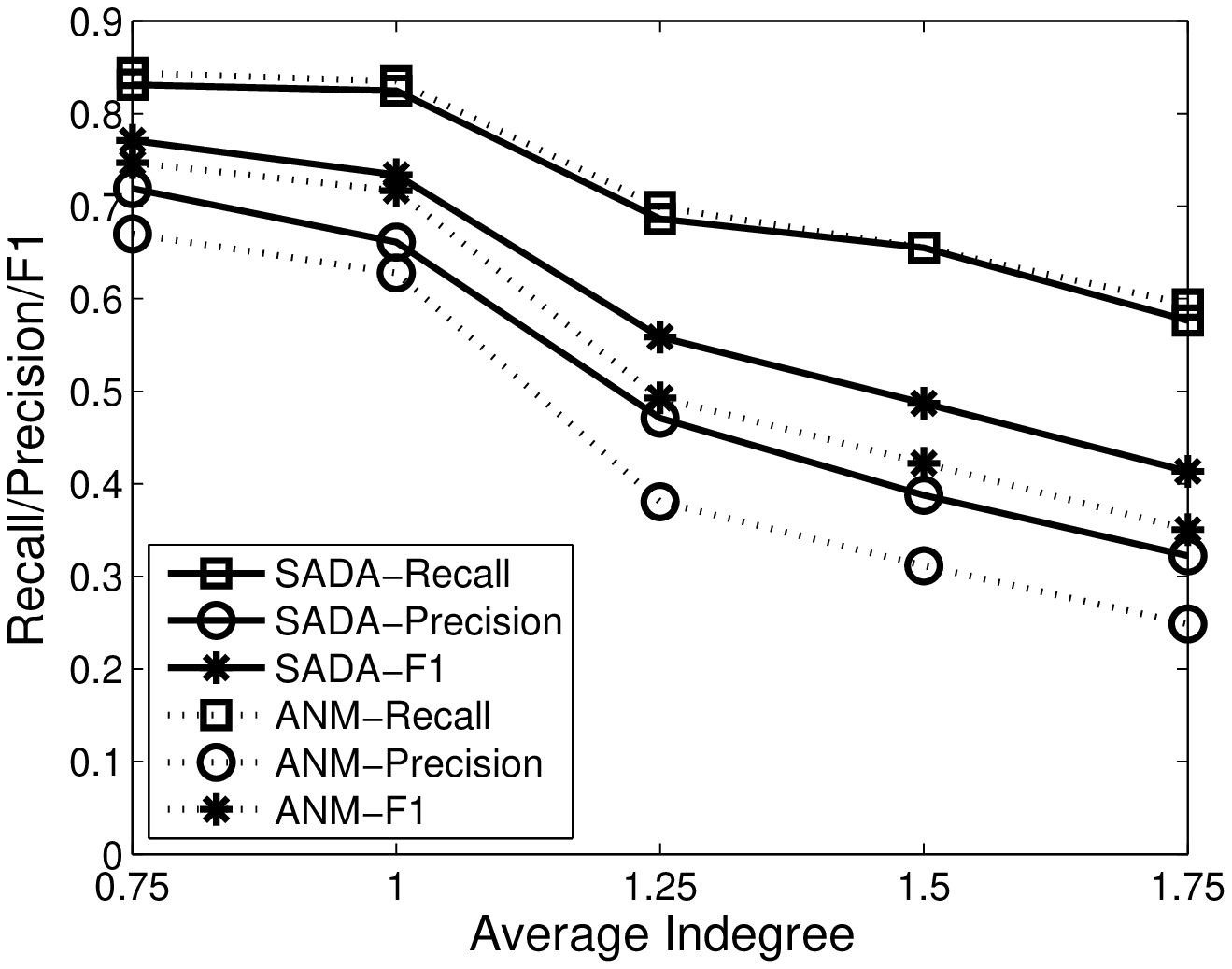}
    \label{fig:sim_sens_con_disc}
  }
  \caption{Sensitivity to the connectivity.}
  \label{fig:sim_con_sample}
\end{figure}

\noindent\textbf{Sensitivity to Noise}

This set of experiments are only conducted on the linear non-Gaussian data, as it is way too difficult to control the noise ratio in the generation of discrete data.
The noise has several effects in SADA, firstly moderate ratio of noise contributes to determined the direction of the causality in the basic causal solvers;
secondly the noise will decrease the reliability of the conditional independence tests and further reduce the quality of the partition; thirdly, too much noise also reduces the quality of the basic causal solvers.
When the noise weight is less than 0.3, both SADA and LiNGAM's performance is insensitive to the noise because of the trade-off between the first two effects.
When the noise ratio is larger than 0.3, both recall and precision reduce with the increasing of noise because of the third effect.


\begin{figure} [h]
  \centering
   \includegraphics[width=0.47\textwidth]{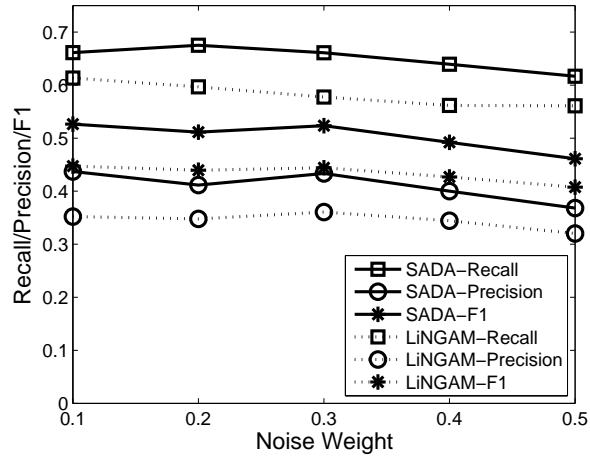}
  \caption{Sensitivity to the noise on linear non-Gaussian model.}
  \label{fig:sim_noise}
\end{figure}

\subsection{Results on Real-World Structures}

It generally, real-world Bayesian network structures cover a variety of applications, including, medicine (\emph{Alarm} dataset), weather forecasting (\emph{Hailfinder} dataset),  printer troubleshooting (\emph{Win95pts} dataset), pedigree of breeding pigs (\emph{Pigs} dataset) and linkage among genes (\emph{Link} dataset). The structural statistics of these Bayesian networks are summarized in \tabref{tab:dataset}. In all the Bayesian networks, the maximal degrees, i.e., the maximal number of parental variables in the networks, are no larger than 6, regardless of the total number of variables. This verifies the correctness of our sparsity assumption. 

\begin{table}[!ht]
\center
\caption{Statistics on the Datasets}
\label{tab:dataset}
\begin{tabular}{|c|c|c|c|}
\hline Dataset & Variable \# & Avg degree & Max degree\\
\hline \emph{Alarm} & 37 & 1.2432 &4 \\
\hline \emph{Hailfinder} &56 & 1.1786 &4 \\
\hline \emph{Win95pts} & 76 &0.9211 & 6 \\
\hline \emph{Pigs} & 441 & 1.3424 & 2 \\
\hline \emph{Link} & 724 & 1.5539 & 3 \\
\hline
\end{tabular}
\end{table}

\noindent\textbf{On Linear Non-Gaussian Model}

The causal cutting errors are reported in Figure \ref{fig:lingam:de}, on varying the number of samples generated by the Bayesian networks. Even when the samples size is $2|V|$, the highest causal cutting error is within 0.12. Moreover, the causal cutting errors generally decrease with the growth of sample size. These results reveal the fundamental advantage of SADA, such that the sufficient number of samples only depends on the sparsity of the causal structure but not the number of variables. Note that the baseline approach LiNGAM does not work when the number of samples are as small as $2|V|$.

\begin {figure} [!t]
\begin {center}
\includegraphics [width=0.47\textwidth]{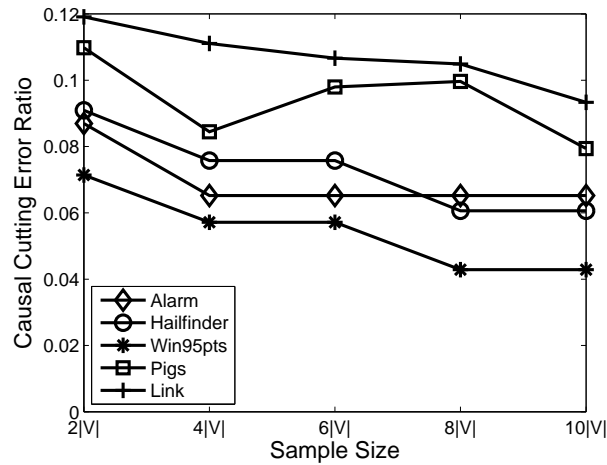}
\caption {causal cutting error ratio on linear non-Gaussian models.}
\label{fig:lingam:de}
\end {center}
\end {figure}

In the following experiments, we compare SADA against the baseline approach by fixing the sample size at $2|V|$. As shown in Table \ref{tab:lingam:rp}, SADA achieves significantly better F1 score on all of the five datasets. SADA is particularly doing well on precision, i.e., returning more accurate causal relations. SADA's division strategy is the main reason behind the improvement of precision on SADA. Specifically, the division on variables allows SADA to remove a large number of candidate variable pairs if they are assigned to $V_1$ and $V_2$. The basic causal solver is run on subproblem of much smaller scale, thus generating more reliable results. The Recall of SADA is comparable to LiNGAM on four of the datasets, and slightly worse on the other one. This shows that the unavoidable causal cutting error does not affect the recall under linear non-Guassian models.


\begin{table*}
\center \caption{Results on Linear Non-Gaussian Model}
\label{tab:lingam:rp}
\begin{tabular}{|c|c|c|c|c|c|c|c|c|c|c|c|c|}
\hline \multirow{2}{*}{Dataset} & \multicolumn{4}{|c|}{Recall} & \multicolumn{4}{|c|}{Precision} & \multicolumn{4}{|c|}{F1 Score}\\
\cline{2-13}       & \tiny{SADA} & \tiny{LiNGAM}  & \tiny{DLiNGAM} & \tiny{SICA} & \tiny{SADA} & \tiny{LiNGAM}  & \tiny{DLiNGAM} & \tiny{SICA}& \tiny{SADA} & \tiny{LiNGAM}  & \tiny{DLiNGAM} & \tiny{SICA} \\
\hline \emph{Alarm}      & \textbf{0.41} & 0.24 &0.02 & 0.20 & \textbf{0.36} & 0.30 & 0.20 & 0.33 & \textbf{0.38} & 0.27&0.04 &0.25 \\
\hline \emph{Hailfinder} & \textbf{0.52} & 0.24 & 0.23& 0.42 &\textbf{0.46} & 0.13 & 0.15 & 0.45&\textbf{0.49} & 0.17&0.18 &0.39 \\
\hline \emph{Win95pts}   & \textbf{0.57} & 0.41 & 0.07 &0.43&   \textbf{0.42} & 0.23 &0.10&0.45& \textbf{0.48} & 0.30 & 0.08 &0.44 \\
\hline \emph{Pigs}       & 0.56 & \textbf{0.57} &N.A.&N.A. & \textbf{0.23} & 0.12 &N.A.&N.A.& \textbf{0.33} & 0.19 &N.A.&N.A.\\
\hline \emph{Link}       & \textbf{0.62} & 0.53 &N.A.&N.A. & \textbf{0.25} & 0.07 &N.A.&N.A.& \textbf{0.36} & 0.13 &N.A.&N.A.\\
\hline
\end{tabular}
\end{table*}
\normalsize

\noindent \textbf{On Discrete Additive Noise Model}

The causal cutting error of SADA on the discrete data is presented in Figure \ref{fig:disc:de}, which shows similar property of the result on linear non-Gaussian models. This further verifies the generality of SADA on different data domains.

\begin {figure} [!t]
\begin {center}
\includegraphics [width=0.47\textwidth]{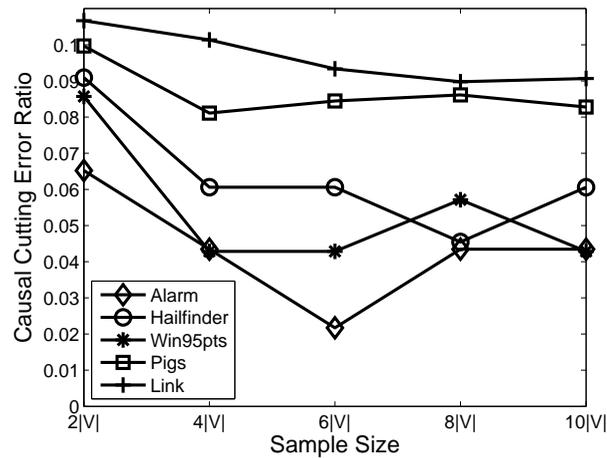}
\caption {Causal cutting error ratio on discrete models.}
\label{fig:disc:de}
\end {center}
\end {figure}

In this group of experiments, we fix the sample size at 2000, and report recall, precision and F1 score in Table \ref{tab:disc:rp}. Note that ANM is only applicable to domain with small number of variables. Because it cannot finish the computation on \emph{Pigs} and \emph{Link} in one week. This proves the improv
ement of SADA on scalability in terms of the variables. Generally speaking, the results in the table also verify the effectiveness of SADA, especially the enhancement on precision and F1 score.

\begin{table}[!ht]
\center
\caption{Results on Discrete Model}
\label{tab:disc:rp}
\begin{tabular}{|c|c|c|c|c|c|c|}
\hline \multirow{2}{*}{Dataset} & \multicolumn{2}{|c|}{Recall} & \multicolumn{2}{|c|}{Precision} & \multicolumn{2}{|c|}{F1 Score}\\
\cline{2-7}       & \tiny{SADA} & \tiny{ANM}  & \tiny{SADA} & \tiny{ANM}  & \tiny{SADA} & \tiny{ANM}  \\
\hline \emph{Alarm}      & \textbf{0.67} & 0.65 & \textbf{0.72} & 0.60 & \textbf{0.70} & 0.63 \\
\hline \emph{Hailfinder} & 0.71 & \textbf{0.76} & \textbf{0.57} & 0.45 & \textbf{0.63} & 0.56 \\
\hline \emph{Win95pts}   & 0.68 & \textbf{0.71} & \textbf{0.41} & 0.38 & \textbf{0.51} & 0.49 \\
\hline \emph{Pigs}       & \textbf{0.68} & N.A. & \textbf{0.50} & N.A. & \textbf{0.58} & N.A. \\
\hline \emph{Link}       & \textbf{0.69} & N.A. & \textbf{0.46} & N.A. & \textbf{0.56} & N.A. \\
\hline
\end{tabular}
\end{table}
\normalsize

As a conclusion, SADA shows excellent performance on 5 different domains with real-world Bayesian networks. SADA returns accurate causal structure when combined with two well known causal inference algorithms. The causal cut used to partition the problem does incur certain error on incorrect partitioning. Despite of the errors, SADA still outperforms ANM without partitioning on almost all settings.


\section{Conclusion}\label{sec:concl}

In this paper, we present a general and scalable framework, called SADA, to support causal structure inference, using a \emph{split-and-merge} strategy. In SADA, causal inference problem on a large variable set is partitioned into subproblems with overlapping subsets of variables, utilizing the concept of causal cut. Our proposal facilitates existing causation discovery algorithms to handle problem domains with more variables and less samples, which extend the application scenarios of causation discovery. Strong theoretical analysis proves the effectiveness, correctness and completeness guarantee of SADA under a general setting. Experimental results further verifies the usefulness of the new framework with two mainstream causation algorithms on linear non-Gaussian model and discrete additive noise model. Theoretical and experimental analysis of SADA reveal the fundamental advantage of our approach, that the required sample depends on the generating graph connectivity and not the size of the variable set; this yields up to exponential savings in sample relative to previously known algorithms.

%

While our methods haven shown improvement over existing methods, we believe there remains room for further enhancement. One possible direction is to attempt other existing randomized division strategies commonly used to tackle combinatorial problems on graph data. Another interesting problem is how to reduce the computational cost when subproblems have a large overlap on variables.



\ifCLASSOPTIONcaptionsoff
  \newpage
\fi

\bibliographystyle{IEEEtran}
\bibliography{IEEEabrv,ref}





\appendices

\section{Pseudocodes of Causal Structure Generation}

\begin{algorithm}[!h]
	\caption{Causal Structure Generator}\label{algo:csg}
	\SetKwFunction{CSGenerator}{\textbf{CSGenerator}}
	\CSGenerator($n$, $d$)\\
	\KwIn {$n$: the number of variable\label{key}s, $d$: the average in-degree.}
	\KwOut {$G$: Causal structure in the formal of boolean adjacency matrix}
	Set $G$ to an $n\times n$ false matrix;\\
	\For {$i = 1$ to $n$}
	{
		Generate a rand integer $n_p$ with mean $d$ in the range $[\lfloor{d} \rfloor, \lceil{d}\rceil]$;\\
		Set $n_p=\min\{n_p, i-1\}$;\\
		Set $j=0$;\\
		\While{$j < n_{p}$}
		{
			Generate a rand integer $k$ in the range $[1, i-1]$;\\
			\If {$G\left(i, k\right)=false$}
			{
				Set $G\left(i, k\right)=true$;\\
				Set $j=j+1$;
			}
		}
	}
\end{algorithm}

\section{Pseudocodes of Linear Non-Gaussian Data Generation}
\begin{algorithm}[!h]
	\caption{Linear Non-Gaussian Data Generator}\label{algo:ldg}
	\SetKwFunction{LiNGAMDataGenerator}{\textbf{LiNGAMDataGenerator}}
	\LiNGAMDataGenerator($CS$, $w$, $n$)\\
	\KwIn {$G$: Causal structure in the formal of $n\times n$ boolean adjacency matrix,$w$: noise weight, $m$: the number of samples.}
	\KwOut {$D$: Generated sample}
	\For {$i = 1$ to $n$}
	{
		$P_{i}$ is the parent variable set of $v_i$ obtained from $G$;\\
		Generate $m\times 1$ rand vector $U$ in $[0, 1]$ with Non-Gaussian distribution;\\
		Normalize $U$ to mean 0 and variance 1.\\
		Set the value of $v_i$, $D_i=w*U$;\\
		\ForEach {$v_j \in P_{i}$}
		{
			Set $D_i = D_i+ D_j$;\\
		}
		Normalize $D_i$ to mean 0 and 1 variance.\\
	}
\end{algorithm}
\end{document}